\documentclass{article} % For LaTeX2e
\usepackage[preprint]{colm2026_conference} % preprint = author shown, no line numbers (arXiv style)

\usepackage[utf8]{inputenc} % allow utf-8 input
\usepackage[T1]{fontenc}    % use 8-bit T1 fonts

\usepackage{amsmath}
\usepackage{amssymb}
\usepackage{amsfonts} 
\usepackage{bm}
\usepackage{mathtools}
\usepackage{amsthm} 
\usepackage{nicefrac}        % compact symbols for 1/2, etc.

\usepackage{booktabs}       % professional-quality tables
\usepackage{array}
\usepackage{graphicx}
\usepackage{subcaption}
\usepackage[table]{xcolor}  % row colors (\rowcolor)
\usepackage{multirow}
\usepackage{placeins}
\usepackage{wrapfig}

\usepackage{enumitem}
\usepackage{microtype}
\usepackage{ragged2e}

\usepackage{algorithm}
\usepackage{algorithmic}

\usepackage[most]{tcolorbox}
\tcbuselibrary{listings}

\usepackage{hyperref}
\usepackage{url}

\usepackage[capitalize,noabbrev]{cleveref}

\usepackage{tikz}
\usetikzlibrary{arrows.meta, positioning, calc, backgrounds, shapes.multipart, fit}

\usepackage{arxiv_style}

\graphicspath{{../}{./}}

\definecolor{darkblue}{HTML}{254E7A} % palette navy (Color 01) -- links
\definecolor{linkpink}{HTML}{D6336C} % deep pink -- URLs
\hypersetup{colorlinks=true, linkcolor=darkblue, urlcolor=linkpink, citecolor=darkblue}
\setlist[itemize]{leftmargin=1.2em, topsep=2pt, itemsep=2pt}
\setlist[enumerate]{leftmargin=1.4em, topsep=2pt, itemsep=2pt}

\definecolor{paletteNavy}{HTML}{254E7A}   % Color 01 -- darkest
\definecolor{paletteBlue}{HTML}{5280AC}   % Color 02 -- mid
\definecolor{paletteSky}{HTML}{7EBCE1}    % Color 03 -- light
\definecolor{palettePale}{HTML}{C5E0EE}   % Color 04 -- very light
\definecolor{paletteCream}{HTML}{F3F2EA}  % Color 05 -- near-white warm

\definecolor{midnightgreen}{rgb}{0.0, 0.29, 0.33}
\definecolor{deepgreen}{HTML}{055c29}
\definecolor{deeppurple}{HTML}{7030a0}
\definecolor{deepblue}{HTML}{171d91}
\definecolor{brown}{HTML}{6B5226}  % appendix umber for highlighted instruction text
\definecolor{shadered}{HTML}{ffe5e5}
\definecolor{shadegreen}{HTML}{e5f7ed}
\definecolor{msftBlack}{RGB}{0,0,0}
\definecolor{lightred}{RGB}{255,163,163}
\definecolor{deepred}{RGB}{153,0,0}
\definecolor{semInputFill}{HTML}{DCEFD8}\definecolor{semInputEdge}{HTML}{8FBE8B}      % green
\definecolor{semRespFill}{HTML}{FFE7BF}\definecolor{semRespEdge}{HTML}{E2A64B}        % amber
\definecolor{semCorrectFill}{HTML}{E8DAF1}\definecolor{semCorrectEdge}{HTML}{B58BC7}  % lavender
\definecolor{semPolicyFill}{HTML}{FFF3C7}\definecolor{semPolicyEdge}{HTML}{D8B23E}    % yellow
\definecolor{semVerifFill}{HTML}{DDEAF7}\definecolor{semVerifEdge}{HTML}{8FB6D9}      % blue
\definecolor{semTrainFill}{HTML}{F8DAD8}\definecolor{semTrainEdge}{HTML}{D97A73}      % salmon
\definecolor{semArrow}{HTML}{6F95C8}\definecolor{semArrowLight}{HTML}{AFC7E8}         % main / branch arrows
\definecolor{semStageBg}{HTML}{F7F8FA}\definecolor{semStageEdge}{HTML}{B8B8B8}        % dashed stage container
\definecolor{semTitle}{HTML}{1F2933}     % section title / label text
\definecolor{semNavy}{HTML}{254C78}      % strong takeaway-strip accent

\definecolor{barblue}{HTML}{8FB6D9}        % policy node  (muted pale blue)
\definecolor{barorange}{HTML}{D97A73}      % phi / tts kernel (salmon -- the active part)
\definecolor{bargreen}{HTML}{E2A64B}       % reward node  (amber)

\lstdefinestyle{pythonstyle}{
  language=Python,
  basicstyle=\ttfamily\small,
  keywordstyle=\color{blue},
  commentstyle=\color{gray},
  stringstyle=\color{orange},
  showstringspaces=false,
  breaklines=true
}

\newsavebox{\fitbox}
\newcommand{\fitwidth}[1]{%
  \savebox{\fitbox}{#1}%
  \ifdim\wd\fitbox>\linewidth
    \resizebox{\linewidth}{!}{\usebox{\fitbox}}%
  \else
    \usebox{\fitbox}%
  \fi
}

\theoremstyle{plain}
\newtheorem{theorem}{Theorem}[section]

\newtheorem{lemma}[theorem]{Lemma}

\theoremstyle{definition}

\theoremstyle{remark}

\newcommand{\SR}{{\phi_{\mathrm{SR}}}}
\newcommand{\PhiR}{\Phi_{\mathcal{R}}}

\newcommand{\Rev}{\mathcal{R}}

\newcommand{\JSR}{J_{\SR}}
\newcommand{\JOneShot}{J_{\mathrm{OneShot}}}

\newcommand{\E}{\mathbb{E}}
\newcommand{\Prob}{\mathbb{P}}

\newcommand{\Xcal}{\mathcal{X}}
\newcommand{\Ycal}{\mathcal{Y}}
\newcommand{\rstar}{r^{\star}}

\newcommand{\Yuanxin}[1]{}
\newcommand{\mgh}[1]{}
\newcommand{\rz}[1]{}
\newcommand{\mh}[1]{}
\floatname{algorithm}{Protocol}

\title{REVES: REvision and VErification--Augmented Training for Test-Time Scaling}
\author{
{\bf Yuanxin Liu$^{1}$\thanks{Work done during an internship at Amazon.}\enspace Ruida Zhou$^{2}$\enspace Xinyan Zhao$^{2}$\enspace Amr Sharaf$^{2}$\enspace Hongzhou Lin$^{2}$} \\
{\bf Arijit Biswas$^{2}$\enspace Mohammad Ghavamzadeh$^{3}$\enspace Zhaoran Wang$^{1}$\enspace Mingyi~Hong$^{2,4}$} \\[5pt]
{\normalfont\normalsize $^{1}$Northwestern University \quad $^{2}$Amazon AGI \quad $^{3}$Qualcomm AI Research \quad $^{4}$University of Minnesota}
}

\begin{document}
\maketitle
\lhead{} % remove the "Preprint. Under review." banner; delete this line to keep it
\renewcommand{\headrulewidth}{0pt} % remove the rule above the header

\begin{abstract}
Test-time scaling via sequential revision has emerged as a powerful paradigm for enhancing Large Language Model (LLM) reasoning. However, standard post-training methods primarily optimize single-shot objectives, creating a fundamental misalignment with multi-step inference dynamics. While recent work treats this as multi-turn reinforcement learning (RL), conventional approaches optimize over the multi-step trajectories directly, failing to further exploit the high-quality mistakes in intermediate steps that model can learn from correcting them. We propose a two-stage iterative framework that alternates between online data/prompt augmentation and policy optimization. By converting the intermediate steps (``near-miss'' answers) in the successful recovery trajectories into decoupled revision and verification prompts, our approach concentrates training on both effective answer transformation and error identification. This approach enables efficient off-policy data generation and reduces the computational overhead of long-horizon sampling compared to standard multi-turn RL. On LiveCodeBench, using publicly available test cases as feedback, we observe gains of +6.5 points over the RL baseline and +4.0 points over standard multi-turn training. Beyond coding, our approach matches the previously reported SOTA result on circle packing while using the smallest base model (4B) and far fewer rollouts than the much larger evolutionary search systems. Math results under ground-truth verification further confirm improved correction ability. It also generalizes to out-of-distribution constraint-satisfaction puzzles such as n\_queens and mini\_sudoku, where correctness is defined entirely by problem constraints. Code is available at \url{https://github.com/yxliu02/REVES.git}.
\end{abstract}

% !TEX root = main.tex
\section{Introduction}

Large language models (LLMs) deployed in challenging settings rarely produce a correct answer on the first attempt; real-world workflows routinely loop back with feedback. A growing family of test-time scaling (TTS) algorithms formalizes this pattern, including sequential revision \citep{madaan2023selfrefineiterativerefinementselffeedback,shinn2023reflexionlanguageagentsverbal}, tree search \citep{inoue2025widerdeeperscalingllm}, evolutionary refinement \citep{lee2025evolvingdeeperllmthinking}, and best-of-$N$ with a verifier \citep{chow2025inferenceawarefinetuningbestofnsampling}. They all share a single primitive: given a problem, a previous attempt, and feedback, produce an improved response.

This raises a basic question: \emph{can we design a post-training algorithm that explicitly improves a model's test-time sequential revision capability?} We focus on sequential revision (SR) because gains on SR transfer to the rest of the family: every revision-using TTS algorithm eventually calls the policy on revise-from-prior-attempt inputs. SR is also a strong TTS algorithm in its own right, matching or surpassing tree search and evolutionary refinement on planning tasks (Figure~\ref{fig:tts_landscape}). Theorem~\ref{thm:whysrinformal} shows that increasing the policy's one-step revision value on SR-induced inputs also raises $J_\phi$ for every revision-using TTS algorithm that covers SR's revision inputs, up to small drift terms.

How, then, do we post-train an LLM to optimize the SR objective $\JSR$? Standard post-training, RLHF~\citep{ouyang2022traininglanguagemodelsfollow}, RLVR, and GRPO~\citep{deepseekai2025deepseekr1incentivizingreasoningcapability}, optimizes a single-shot expected reward, which is fundamentally misaligned with the multi-step nature of test-time deployment. The natural fix is multi-turn RL on revision rollouts, but trajectory-level credit broadcasts assign path-dependent credit. Take a wrong, wrong, correct rollout: every per-turn log-probability gradient receives the same positive credit, including the wrong intermediates. Per rollout this credit is biased, and only cancels in expectation across many rollouts.

We propose REVES, a two-stage iterative framework that targets the SR objective $\JSR$ directly through its per-state structure. We show that $\JSR$ decomposes exactly into a weighted sum of per-state one-step recovery probabilities along visited rollouts (Lemma~\ref{lem:hazard}), exposing single-state, single-turn gradients with no horizon credit assignment. REVES is an offline realization of this per-state training signal: each epoch, it runs SR rollouts under the current policy, retains those that succeed within budget, converts the intermediate (``near-miss'') answers into decoupled \emph{revision} and \emph{verification} prompts, and trains with standard single-turn RL on the augmented prompt set; the next epoch refreshes the augmentation. This concentrates training on per-state recovery and avoids both the bias of trajectory broadcasts and the cost of long-horizon online sampling.

Empirically, on LiveCodeBench with public test cases as feedback, REVES achieves $+6.5$ points over the single-shot RL baseline and $+4.0$ points over standard multi-turn training. On the circle packing benchmark, a Qwen3-4B base trained with REVES matches the previously reported best results from much larger evolutionary search systems built on Gemini-2.0 Pro/Flash and Qwen3-8B. Math benchmarks (MATH500, AIME24/25) show consistent improvements under both ground-truth and self-confidence stopping, and the trained policies generalize to out-of-distribution constraint-satisfaction puzzles (n\_queens, mini\_sudoku). Beyond standalone SR, REVES improves every revision-using TTS algorithm we tested (MCTS, AB-MCTS variants, Mind Evolution), in line with the conditional bound in Theorem~\ref{thm:whysrinformal}.

Our contributions are threefold:
\begin{enumerate}[leftmargin=*]
    \item \textbf{Conceptual.} We formulate test-time scaling as a meta-RL problem and show that improving sequential revision also improves any revision-using TTS algorithm that covers SR's revision inputs (Theorem~\ref{thm:whysrinformal}).
    \item \textbf{Methodological.} We exactly decompose $\JSR$ into per-state one-step recovery probabilities (Lemma~\ref{lem:hazard}) and propose REVES, a two-stage framework that recasts intermediate near-miss answers as decoupled revision and verification prompts trained with single-turn RL, avoiding the path-dependent credit assignment of multi-turn RL.
    \item \textbf{Empirical.} REVES delivers consistent test-time gains across math, coding, out-of-distribution puzzles, and circle packing, and improves every revision-using TTS algorithm we evaluated.
\end{enumerate}

% !TEX root = main.tex
\section{Related Work}

\paragraph{Test-time-aware post-training.}
A growing line of work shows the importance of aligning the post-training objective with the intended test-time inference strategy. For parallel-sampling families, this includes optimizing pass@$k$~\citep{chen2025passktrainingadaptivelybalancing, walder2025passkpolicyoptimizationsolving, tang2025optimizinglanguagemodelsinference} when deployment uses best-of-$N$, majority voting~\citep{wang2023selfconsistencyimproveschainthought}, or learned verifiers~\citep{chow2025inferenceawarefinetuningbestofnsampling}. Our deployment is structurally different: the test-time strategies we care about, including sequential revision, MCTS variants, and Mind Evolution, are all \emph{revision-using} algorithms whose calls take as input a previous attempt together with its feedback. Aligning training with this inference family therefore requires a different objective. 

\paragraph{Training sequential revision capability.}
Closer to our setting, several prior methods explicitly train the policy to revise, including supervised and preference-based approaches~\citep{qu2024recursiveintrospectionteachinglanguage, xiong2025buildingmathagentsmultiturn}, RL-based approaches~\citep{kumar2024traininglanguagemodelsselfcorrect, jiang2025pagmultiturnreinforcedllm, lee2025reviselearningrefinetesttime}, and multi-turn RL methods that incorporate per-turn natural-language or numerical critiques~\citep{jain2025multiturncodegenerationsinglestep, li2025learningreasonfeedbacktesttime, li2025lanpobootstrappinglanguagenumerical, zhang2026critiquegrpoadvancingllmreasoning}. Our work shares the broader goal of strengthening the model's revision capability, but differs in framing: we primarily focus on how to align training with the test-time objective of a revision-using inference strategy.

\paragraph{Weakness-driven data synthesis and guided exploration.}
SwS~\citep{liang2025swsselfawareweaknessdrivenproblem} synthesizes new problems targeting the model's weaknesses using strong external teacher models, and POPE~\citep{qu2026popelearningreasonhard} uses privileged oracle hints to guide on-policy exploration on hard problems. These methods are orthogonal to ours: they target single-shot pass@1, whereas REVES targets a fundamentally different test-time objective.

% !TEX root = main.tex
\section{Test-Time Scaling as a Meta-RL Problem}
\label{sec:meta-rl-formulation}

\paragraph{Setup and two motivating examples.}
We fix notation used throughout the paper. Denote by $\Xcal$ the distribution
over input problems, $\Ycal$ the answer space, and
$r^\star : \mathcal{X} \times \mathcal{Y} \to [0,1]$ the \emph{ground-truth}
reward, where $r^\star(x, y) = 1$ indicating a fully correct answer. %; partial credit within $(0,1)$ is allowed (the binary case $r^\star \in \{0,1\}$ is recovered as a special case). 
We focus on the verifiable-reward regime
where $r^\star$ is computable during both training and at test time, e.g.,
public test-case execution for code, and
constraint satisfaction for puzzles. When $r^\star$ is unavailable at test
time, a surrogate stopping rule replaces it
(Section~\ref{sec:method}). We model the LLM as a policy $\pi_\theta$ in RL parameterized by $\theta$. At deployment, an
LLM rarely operates by a single call to $\pi_\theta$; instead a
\emph{test-time scaling (TTS) algorithm} $\phi$ orchestrates up to $K$
policy calls and extracts a final answer $\hat y \in \Ycal$. Two canonical
examples bracket the design space:
\emph{Best-of-$N$ (BoN)}~\citep{chow2025inferenceawarefinetuningbestofnsampling}
draws $N$ candidates in parallel from $\pi_\theta(\cdot \mid x)$ and selects
$\hat y$ via an external scorer; every call uses the same
original prompt.
\emph{Sequential Revision
(SR)}~\citep{madaan2023selfrefineiterativerefinementselffeedback} generates
an initial response $y_1 \sim \pi_\theta(\cdot \mid x)$ from prompt $x$, and
for $t = 2, \ldots, K$ produces a revision
$y_t \sim \pi_\theta(\cdot \mid x, y_{t-1}, f_{t-1})$ conditioned on the
previous attempt $y_{t-1}$ and a feedback signal $f_{t-1}$ (e.g.,
compiler error messages in code generation), halting at the first correct
response. Detailed pseudocode is in
Protocol~\ref{alg:seqrev_setup} of Appendix~\ref{appendix: setup}.

\paragraph{General formulation: the TTS-induced decision process.}
Pairing any TTS algorithm $\phi$ with $\pi_\theta$ induces a decision process with prompt-state space $\mathcal{S}$ and response space $\Ycal$. At each step $t$, the policy emits a response $y_t \sim \pi_\theta(\cdot \mid s_t)$, and $\phi$ constructs the next prompt via a transition kernel $P_\phi(s_{t+1} \mid s_{1:t}, y_{1:t})$. The initial state is $s_1 = x$. Note that the policy itself remains Markov in the current state (at every step, $\pi_\theta$ conditions only on the local prompt $s_t$); the history dependence is in the kernel, because $\phi$ may branch from any earlier response. We focus on the family
$\PhiR$ of TTS algorithms that issue at least one revision call on a
prior attempt; this includes
Sequential Revision and adaptive
tree search such as AB-MCTS~\citep{inoue2025widerdeeperscalingllm}. For such algorithms a
revision-step prompt has the form
$
s_{t+1} = (x, y_{i_t}, f_{i_t}),\ i_t \in \{1, \ldots, t\},
$
where the prior step $i_t$ is chosen by $\phi$: for SR, $i_t = t$ (always the most recent response), so $s_{t+1}$ depends only on $(s_t, y_t)$ and the induced process is a standard MDP; for tree-search algorithms such as AB-MCTS, $i_t$ may be any earlier response that $\phi$'s search policy decides to branch from, and the kernel is genuinely history-dependent. The procedure runs until a random stopping
time $\tau := K \land \min \{ t : r^\star(x, y_t) = 1 \}$ and outputs $\hat y = y_\tau$; that is, the procedure either \emph{early-stops} once a verified-correct response is produced, or \emph{exhausts the budget} $K$. Running
$\pi_\theta$ inside $\phi$ yields the TTS-induced objective
\[
J_\phi(\theta) = \E_{x \sim \Xcal,\, (s_{1:\tau}, y_{1:\tau}) \sim (\pi_\theta, \phi)}\bigl[r^\star(x, y_\tau)\bigr],
\]
to be contrasted with the standard single-shot objective
\[
\JOneShot(\theta) = \E_{x \sim \Xcal,\, y \sim \pi_\theta(\cdot \mid x)}\bigl[r^\star(x, y)\bigr],
\]
which is what RLHF~\citep{ouyang2022traininglanguagemodelsfollow}, RLVR, and
GRPO~\citep{deepseekai2025deepseekr1incentivizingreasoningcapability}
optimize. Optimizing $J_\phi$ rather than $\JOneShot$ is a \emph{meta-RL}
problem (Figure~\ref{fig:rl-vs-metarl}): the policy must adapt across the
distribution of multi-step contexts that $\phi$ generates, not just answer
the original question. Our goal is to improve $J_\phi$ for every TTS algorithm
$\phi \in \PhiR$.

\begin{figure}[t]
\centering
\resizebox{\textwidth}{!}{%
\begin{tikzpicture}[
  font=\small,
  % --- LEFT (Standard RL): muted / washed-out, to read as "simple, single-shot"
  mstate/.style={                       % gray state
    draw=black!40, rounded corners=2pt, thick,
    minimum width=0.75cm, minimum height=0.65cm, fill=black!6
  },
  mpolicy/.style={                      % pale, desaturated blue policy
    draw=semVerifEdge!75!gray, rounded corners=2pt, thick,
    minimum width=0.9cm, minimum height=0.65cm, fill=semVerifFill!55!white
  },
  mreward/.style={                      % pale, desaturated amber reward
    draw=semRespEdge!75!gray, rounded corners=2pt, thick,
    minimum width=0.8cm, minimum height=0.65cm, fill=semRespFill!55!white
  },
  % --- RIGHT (Meta-RL): vivid. One UNIFIED light-blue chain for every state and
  %     response (so the response is clearly NOT the thing getting complex)...
  chain/.style={                        % response y_t (blue, like the left's y)
    draw=semVerifEdge, rounded corners=2pt, thick,
    minimum width=0.85cm, minimum height=0.65cm, fill=semVerifFill
  },
  cstate/.style={                       % state s_t (grey, like the left's x)
    draw=black!40, rounded corners=2pt, thick,
    minimum width=0.85cm, minimum height=0.65cm, fill=black!6
  },
  % ...and the orchestration kernel phi in SALMON, the one active/complex piece.
  phi/.style={
    draw=semTrainEdge, rounded corners=2pt, very thick,
    minimum width=0.8cm, minimum height=0.65cm, fill=semTrainFill
  },
  reward/.style={                       % amber reward (right side, vivid)
    draw=semRespEdge, rounded corners=2pt, thick,
    minimum width=0.8cm, minimum height=0.65cm, fill=semRespFill
  },
  arr/.style={-{Latex[length=2mm]}, thick, semArrow},
  marr/.style={-{Latex[length=2mm]}, thick, black!45},   % muted arrows on the left
  captxt/.style={font=\footnotesize\itshape, semTitle}
]
\begin{scope}[shift={(0, 0)}]
  \node[captxt] at (1.7, 0.55) {(a) Standard RL: single call};
  \node[mstate]  (rl_x) at (0.2, -0.4) {$x$};
  \node[mpolicy] (rl_y) at (1.6, -0.4) {$y$};
  \node[mreward] (rl_r) at (3.2, -0.4) {$r^\star(x, y)$};
  \draw[marr] (rl_x.east) -- node[above, font=\scriptsize, text=black] {$\pi_\theta$} (rl_y.west);
  \draw[marr] (rl_y.east) -- (rl_r.west);
  \node[font=\scriptsize, align=center, text=black!55] at (1.7, -1.4)
    {one prompt, one response, one reward};
\end{scope}
\draw[gray!50, thick, dashed] (3.96, 0.7) -- (3.96, -2.6);
\begin{scope}[shift={(4.6, 0)}]
  \node[captxt] at (4.7, 0.55) {(b) Meta-RL inside $\phi$: $K$ calls};
  \node[cstate]  (s1)  at (0,    -0.4) {$s_1\!=\!x$};
  \node[chain]   (y1)  at (1.6,  -0.4) {$y_1$};
  \node[cstate]  (s2)  at (2.9,  -0.4) {$s_2$};
  \node[chain]   (y2)  at (4.3,  -0.4) {$y_2$};
  \node (dots) at (5.4, -0.4) {$\cdots$};
  \node[chain]   (yK)  at (6.5,  -0.4) {$\hat y$};
  \node[reward]  (r)   at (8.3,  -0.4) {$r^\star(x, \hat y)$};
  \draw[arr] (s1.east) -- node[above, font=\scriptsize, text=black] {$\pi_\theta$} (y1.west);
  \draw[arr] (y1.east) -- (s2.west);
  \draw[arr] (s2.east) -- node[above, font=\scriptsize, text=black] {$\pi_\theta$} (y2.west);
  \draw[arr] (y2.east) -- (dots.west);
  \draw[arr] (dots.east) -- (yK.west);
  \draw[arr] (yK.east) -- (r.west);
  \node[phi] (phi2) at (2.9, -1.6) {$\phi$};
  \node[phi] (phiK) at (6.5, -1.6) {$\phi$};
  \draw[arr] (phi2.north) -- (s2.south);
  \draw[arr] (phiK.north) -- (yK.south);
  % history feeds into phi (drawn on background layer so they pass behind boxes).
  % phi_2 (early step) gets explicit arrows from its only history (s_1, y_1).
  % phi_K (later step) gets one explicit arrow from y_2 plus a stub arrow
  % representing "earlier history" without connecting to specific boxes.
  \begin{scope}[on background layer]
    \draw[-{Latex[length=2mm]}, thick, dashed, semArrowLight!60!black]
      (s1.south) to[out=270, in=180] (phi2.west);
    \draw[-{Latex[length=2mm]}, thick, dashed, semArrowLight!60!black]
      (y1.south) to[out=270, in=120] (phi2.north west);
    \draw[-{Latex[length=2mm]}, thick, dashed, semArrowLight!60!black]
      (y2.south) to[out=270, in=120] (phiK.north west);
    % "earlier history" stub arrow into phi_K, not anchored to any box
    \draw[-{Latex[length=2mm]}, thick, dashed, semArrowLight!60!black]
      (5.0, -1.6) -- (phiK.west);
  \end{scope}
  \node[font=\scriptsize, align=center, text=black!55] at (4.7, -2.4)
    {$\phi$ uses the full history $(s_{1:t}, y_{1:t})$; same $\pi_\theta$ at every call.};
\end{scope}
\end{tikzpicture}%
}
\caption{\textbf{(a)} Standard RL: policy $\pi_\theta$ maps problem $x$ to response $y$ with reward $\rstar(x, y)$. \textbf{(b)} A TTS algorithm $\phi$ invokes the same $\pi_\theta$ over up to $K$ steps, each prompt $s_t$ built from the prior history (dashed arrows), and extracts a final response $\hat y = y_\tau$ at stopping time $\tau \le K$; $\pi_\theta$ must therefore behave well across the whole context distribution induced by $\phi$.}
\label{fig:rl-vs-metarl}
\end{figure}
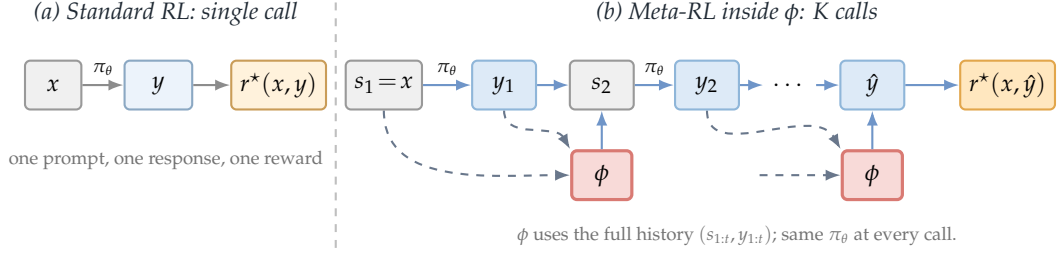

This raises an important question: \emph{does there exist a training target $J_{\phi_{\mathrm{train}}}$ whose optimization provably improves $J_\phi$ for every $\phi \in \PhiR$?} If so, we can train against one algorithm and have the gains transfer to every other deployment-time choice in $\PhiR$, side-stepping the cost of running tree search or external verifiers during training. The answer is affirmative. The key fact is that every algorithm in $\PhiR$ at some point invokes $\pi_\theta$ to revise a previous attempt, so raising the revision value on a training target's inputs improves every member of $\PhiR$ that covers those inputs.

\begin{theorem}[Sequential-revision recovery transfers to revision-using TTS]
\label{thm:whysrinformal}
Let $\pi_0$ be a baseline policy and $\pi_1$ an updated policy. For $\phi \in \PhiR$, policy $\pi$, and revision-call input $z$, let
\[
V_\pi(z) := \Prob_{y \sim \pi(\cdot \mid z)}\bigl(\rstar(x, y) = 1\bigr), \qquad
\rho^\phi_\pi(z) := \E_{(\phi, \pi)}\!\Bigl[\textstyle\sum_{t=1}^{\tau} \mathbf{1}\{z_t = z\}\Bigr]
\]
denote the one-step recovery probability and the expected visit count to $z$. Assume:
\begin{enumerate}[label=\textnormal{(C\arabic*)}, leftmargin=2.2em, labelsep=0.4em, labelwidth=1.8em, align=left, noitemsep, topsep=2pt]
    \item \emph{Coverage of SR by $\phi$ at $\pi_0$:} there exists $C_\phi \ge 1$ such that $\rho^\phi_{\pi_0}(z) \ge \rho^\SR_{\pi_0}(z)/C_\phi$ for every $z$.
    \item \emph{Monotone recovery improvement:} $V_{\pi_1}(z) \ge V_{\pi_0}(z)$ for every $z \in \mathrm{supp}(\rho^\phi_{\pi_0})$.
\end{enumerate}
Then, we have the following guarantee:
\[
J_\phi(\pi_1) - J_\phi(\pi_0) \;\ge\; \frac{1}{C_\phi}\bigl[\JSR(\pi_1) - \JSR(\pi_0)\bigr] \;-\; \bigl\|\rho^\phi_{\pi_1} - \rho^\phi_{\pi_0}\bigr\|_1 \;-\; \frac{1}{C_\phi}\bigl\|\rho^\SR_{\pi_1} - \rho^\SR_{\pi_0}\bigr\|_1.
\]
\end{theorem}

See Appendix~\ref{appendix:proof-transfer} for the proof.

\paragraph{Our training target $\JSR$.}
The theorem licenses any $\phi_{\mathrm{train}} \in \PhiR$ as a training target, with the transfer guarantee holding for every deployment-time $\phi$ that covers it. We pick $\SR$ and optimize
\[
\JSR(\theta) := \E_{x \sim \Xcal,\, y_{1:\tau} \sim (\pi_\theta, \SR)}\bigl[\rstar(x, y_\tau)\bigr]
\]
for two reasons. First, $\SR$ is itself a strong TTS algorithm. Figure~\ref{fig:tts_landscape} shows that on TravelPlanner~\citep{xie2024travelplannerbenchmarkrealworldplanning}, $\SR$ matches or surpasses Mind Evolution~\citep{lee2025evolvingdeeperllmthinking} and AB-MCTS~\citep{inoue2025widerdeeperscalingllm} (left), with the Markov-revision form $\pi_\theta(\cdot \mid x, y_{\mathrm{prev}}, f)$ on the best cost-accuracy frontier (right), while being the simplest to implement. Second, optimizing $\SR$ transfers to the rest of $\PhiR$: by Theorem~\ref{thm:whysrinformal}, improving one-step recovery on SR's revision inputs raises $J_\phi$ for every revision-using $\phi$ that covers them.

\begin{figure}[t]
    \centering
    \includegraphics[width=0.85\linewidth]{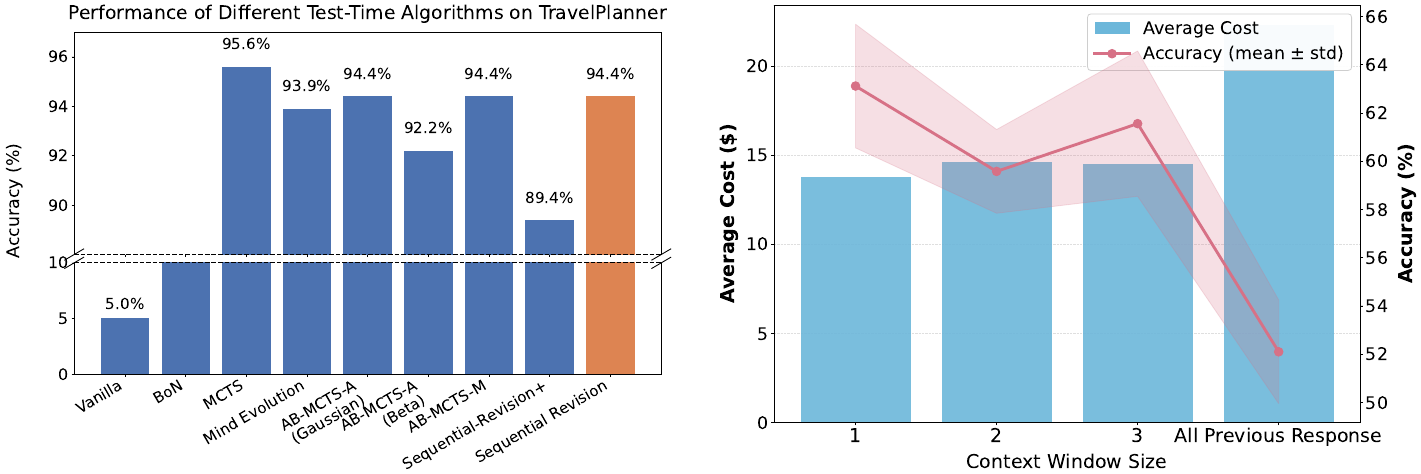}
    \caption{DeepSeek-V3 on TravelPlanner. \textbf{Left:} sequential revision vs.\ Mind Evolution / MCTS / AB-MCTS. \textbf{Right:} cost-accuracy as the revision context window covers the $k$ most recent responses, $k \in \{1, 2, 3\}$, or the full history.}
    \label{fig:tts_landscape}
\end{figure}

\paragraph{Optimizing $\JSR$ is not optimizing single-shot pass@1.}
A natural question is whether standard single-shot RLVR training already addresses $\JSR$. It does not: two policies can be \emph{indistinguishable under $\JOneShot$} on every problem in $\Xcal$ yet differ substantially under $\JSR$.

\begin{theorem}[Objective mismatch]
  \label{thm:mismatch}
  For any revision length $K \geq 2$, there exist policies $\pi_1, \pi_2$, $\Delta(K) > 0$, and a problem distribution $\Xcal$ such that
  $
  \JOneShot(\pi_1) = \JOneShot(\pi_2)$, while $
  \JSR(\pi_2) - \JSR(\pi_1) \geq \Delta(K).
  $
\end{theorem}

The implication is direct: $\JOneShot$ and $\JSR$ are different objectives, and policies with identical single-shot performance can have substantially different sequential-revision performance. Our problem reduces to a single concrete question: how do we post-train $\pi_\theta$ so that $\JSR$ goes up? The next section presents our algorithm.

\begin{tightbox}{paletteNavy}{Takeaway: we adopt the sequential-revision objective $\JSR$ as our training target.}
Improving test-time performance is not the same as optimizing single-shot pass@1 (Theorem~\ref{thm:mismatch}). Among test-time scaling algorithms, sequential revision is both empirically strong and the one whose gains transfer to every other revision-using algorithm (Theorem~\ref{thm:whysrinformal}), which makes it the right objective to optimize.
\end{tightbox}

% =====================================================================
% Material commented out for length; restore if page budget permits.
% =====================================================================
%
% \paragraph{Generation and revision primitives.}
% Every TTS algorithm we are aware of interacts with $\pi_\theta$ through one of
% two primitives:
% \[
% G_\pi(\cdot \mid x) = \pi(\cdot \mid x),
% \qquad
% R_\pi(\cdot \mid x, y_{\mathrm{prev}}, f) = \pi(\cdot \mid x, y_{\mathrm{prev}}, f),
% \]
% i.e., \emph{generation} from scratch and \emph{revision} of a prior attempt
% $y_{\mathrm{prev}}$ conditioned on feedback $f$. Each primitive call takes only
% a local input (either $x$ for $G$, or $(x, y_{\mathrm{prev}}, f)$ for $R$); the
% TTS algorithm's control flow may depend on the full history, but the policy
% itself is queried in a Markov-revision fashion.
% Let $n_G(\phi, K)$ and $n_R(\phi, K)$ denote the expected number of $G$- and
% $R$-calls in a $K$-step run of $\phi$.
% \begin{center}
% \small
% \begin{tabular}{lcc}
% \toprule
% $\phi$ & $n_G(\phi, K)$ & $n_R(\phi, K)$ \\
% \midrule
% One-shot & 1 & 0 \\
% Sequential revision ($\phi_{\mathrm{SR}}$) & 1 & $K-1$ \\
% Best-of-$N$, Self-Consistency & $K$ & 0 \\
% MCTS, AB-MCTS, Mind Evolution & $\ge 1$ & $\ge 1$ \\
% \bottomrule
% \end{tabular}
% \end{center}
% We denote by $\PhiR := \{\phi : n_R(\phi, K) \ge 1 \text{ for } K \ge 2\}$
% the family of \emph{revision-invoking} TTS algorithms.

% !TEX root = main.tex
\section{REVES: From Path-Level Credit to Recovery-Level Supervision}
\label{sec:method}

Section~\ref{sec:meta-rl-formulation} established that improving $\JSR$ is our focus. This section turns that target into an algorithm. We start by writing $\JSR$ exactly as a sum of per-state one-step recovery probabilities along visited trajectories; this hazard decomposition makes the natural training signal explicit and reads off the algorithm directly from the figure.

\begin{lemma}
\label{lem:hazard}
Let $\zeta = (z_1, y_1, \ldots, z_\tau, y_\tau) \sim (\pi_\theta, \SR)$ be a stopped SR trajectory under $\pi_\theta$: at each step $t$ the model samples $y_t \sim \pi_\theta(\cdot \mid z_t)$ and stops at the first verified correct response or at budget $K$, with $z_t = (x, y_{t-1}, f_{t-1})$ the revision state before generating $y_t$. Let
$
V_\pi(z_t) := \Prob_{y' \sim \pi_\theta(\cdot \mid z_t)}\!\big(\rstar(x, y') = 1\big)
$
denote the one-step recovery probability at $z_t$, where $y'$ is a fresh sample (not the trajectory's own $y_t$), and let
$
\rho_\theta(z) := \E_{x \sim \Xcal,\, \zeta \sim (\pi_\theta, \SR)}\!\Big[\sum_{t=1}^{\tau} \mathbf{1}\{z_t = z\}\Big]
$
denote the expected number of times state $z$ is visited along a SR rollout. Then
\begin{align}
\JSR(\theta) = \sum_{t=1}^{K} \E_{x, \zeta}\!\big[\mathbf{1}\{\tau \ge t\}\, V_\pi(z_t)\big] \notag &= \E_{x \sim \Xcal,\, \zeta \sim (\pi_\theta, \SR)}\!\left[\sum_{t=1}^{\tau} V_\pi(z_t)\right] \\
&= \sum_{z}  \rho_\theta(z)\, \E_{y' \sim \pi_\theta(\cdot \mid z)}\!\big[\rstar(x, y')\big].
\label{eq:hazard-decomp}
\end{align}
\end{lemma}

See Appendix~\ref{appendix:proof-hazard} for the proof. Equation~\eqref{eq:hazard-decomp} reads the algorithm off Figure~\ref{fig:reves-illustrative}~(b). The horizontal SR trajectory provides the outer samples $z \sim \rho_\theta$, and at each visited state $z = (x, y_{t-1}, f_{t-1})$ the branches drawn from $\pi_\theta(\cdot \mid z)$ are online rollouts that estimate and optimize the per-state factor $V_\pi(z) = \E_{y' \sim \pi_\theta(\cdot \mid z)}[\rstar(x, y')]$. Since $\JSR$ is a nonnegative-weighted sum of these factors, raising $V_\pi$ at any visited $z$ directly raises $\JSR$, without horizon credit assignment.

% \paragraph{Comparison with multi-turn RL.}
% Multi-turn RL rolls out a revision trajectory and broadcasts a single trajectory-level advantage (the terminal verifier reward, possibly normalized by a group baseline) to every per-turn log-probability gradient. Take a concrete wrong, wrong, correct rollout: all three log-probability gradients receive the same positive advantage, including the two wrong-intermediate steps that did not themselves produce the correct answer. Per rollout this misallocates credit: the early wrong intermediates are reinforced for a downstream success they did not cause. Aggregating across many rollouts cancels the per-rollout variance, but the per-turn credit itself remains biased.
% Our recovery-level form works the other way around: by conditioning on a state $z$ and drawing $N$ fresh samples there, $\widehat V_\pi(z)$ is unbiased and low-variance. The tradeoff is on coverage: a state $z$ enters the gradient only in proportion to $\rho_\theta(z)$, so states the policy rarely visits receive little training signal even when the per-state estimator is sharp. 

\paragraph{Comparison with multi-turn RL.}
Multi-turn RL provides an unbiased policy-gradient estimator for the
trajectory-level objective. However, its per-turn learning signal is
highly indirect. Consider a wrong--wrong--correct rollout: the gradients
for all three turns are weighted by the same positive terminal advantage,
including the two turns that produced incorrect intermediate responses.
This is valid for optimizing terminal success of the whole trajectory,
but it yields a high-variance and coarse signal for improving the local
revision behavior at each intermediate state.

Our method instead reuses such incorrect intermediate responses as
revision states. Given a state
$
z=(x,y_{\mathrm{prev}},f_{\mathrm{prev}}),
$
we draw fresh samples from \(\pi(\cdot\mid z)\) and estimate the local
one-step recovery probability
$
V_\pi(z)
=
\Pr_{y\sim\pi(\cdot\mid z)}
\bigl[r^\star(x,y)=1\bigr].
$
Thus, the intermediate wrong response is not treated as an action to
imitate; it is treated as a state from which the model should learn to
recover. This converts a successful long-horizon trajectory into several
short-horizon revision problems, providing a sharper training signal for
local recovery while relying on the current policy to discover useful
intermediate states.

\paragraph{From recovery target to algorithm.}
An ideal on-policy implementation of Lemma~\ref{lem:hazard} draws $z \sim \rho_\theta$ by running SR and runs single-turn RL at each $z$. Three design choices turn this into the algorithm we use.

\emph{(1) Verification prompts alongside revision data.} At test time the verifier is typically unavailable. We therefore include both \emph{revision prompts} and \emph{verification prompts} (asking the policy to judge whether an intermediate response is correct) in the augmented data, so the trained policy can self-stop at deployment. The split of gains across the two prompt types is in Table~\ref{tab:ablation-verification}.

\emph{(2) Offline generation.} An on-policy estimator of $\nabla \JSR$ would need fresh SR rollouts every gradient step, which is intrinsically serial and prohibitively slow. We instead generate visited states once per epoch in Stage~I, reuse them through Stage~II, and refresh at the next epoch.

\emph{(3) Successful-trajectory filtering.} For a far-from-optimal base policy, online samples at states from failed rollouts rarely produce a verifier-accepted recovery and waste tokens. We restrict Stage~I to states on trajectories that succeed within budget $K$, where at least one recovery target exists by construction; the filter loosens as $\pi_\theta$ improves.

\begin{figure*}[t]
\centering
\includegraphics[width=0.95\linewidth]{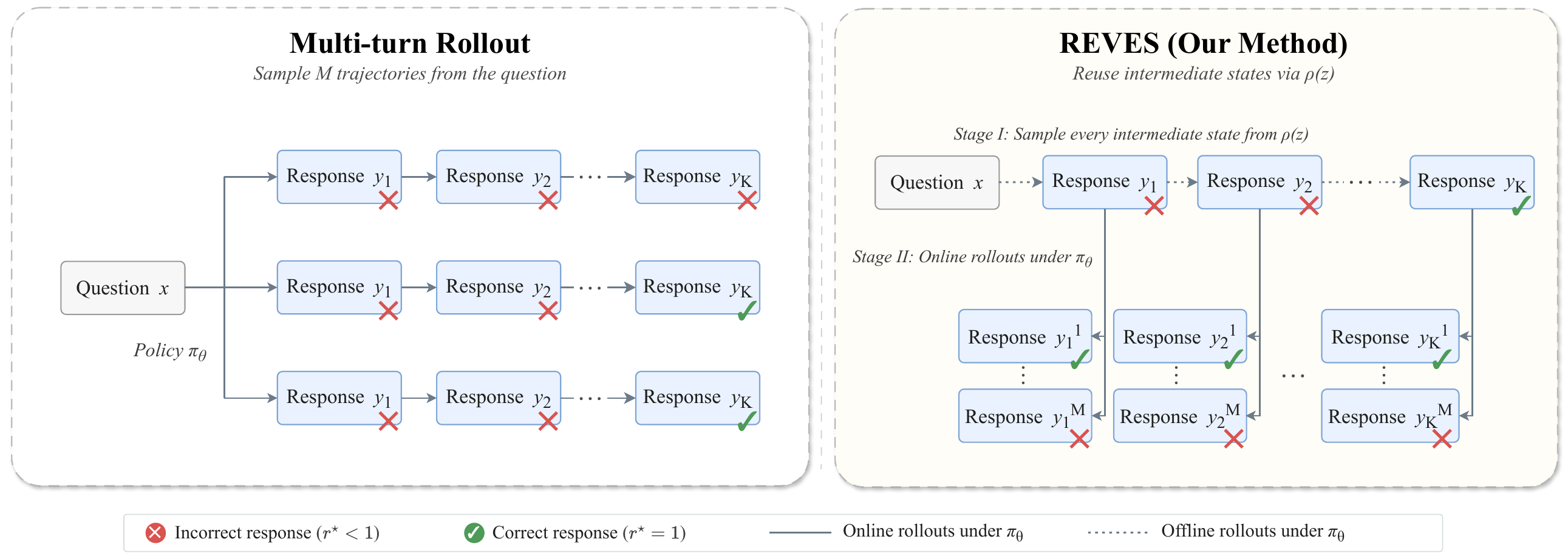}
\caption{Mapping Eq.~\eqref{eq:hazard-decomp} onto the figure. \textbf{(a)} Multi-turn RL training rollouts under the current $\pi_\theta$. \textbf{(b)} The horizontal hatched chain is a retained sequential-revision trajectory; it determines the visit-count weight $\rho_\theta(z)$ at each intermediate state $z = (x, y_{t-1}, f_{t-1})$. The branches at each visited $z$ are fresh samples from $\pi_\theta(\cdot \mid z)$; their verifier outcomes provide the per-state estimate of $V_\pi(z) = \E_{y' \sim \pi_\theta(\cdot \mid z)}[\rstar(x, y')]$.}
\label{fig:reves-illustrative}
\end{figure*}

\paragraph{Two-stage realization.}
\textbf{Stage~I (Data augmentation).} For each input $x$, run SR under the current $\pi_\theta$: $y_1 \sim \pi_\theta(\cdot \mid x)$ and $y_t \sim \pi_\theta(\cdot \mid x, y_{t-1}, \texttt{Feedback}(y_{t-1}))$ for $t = 2, \ldots, K$, terminating at the first correct response or after $K$ attempts. Discard trajectories that fail or are already correct at $y_1$. From each retained trajectory $(x, y_1, \ldots, y_\ell)$ with $\rstar(x, y_\ell) = 1$, build (i) \emph{revision data} $\{\texttt{Revision Prompt}(x, y_i)\}_{i=1}^{\ell}$, where each prompt asks the model to revise the previous response, and (ii) \emph{verification data} $\{\texttt{Verification Prompt}(x, y_i)\}_{i=1}^{\ell}$, where each prompt asks the model to judge whether $y_i$ is correct. Templates are in Appendix~\ref{appendix: setup}. \textbf{Stage~II (Single-turn RL).} Train $\pi_\theta$ on the union of the original RL prompts and the prompts produced in Stage~I. The updated policy serves as the rollout policy for the next epoch's Stage~I. Figure~\ref{fig:alg_overview} shows the loop.

\begin{figure*}[t]
    \centering
    \includegraphics[width=\linewidth]{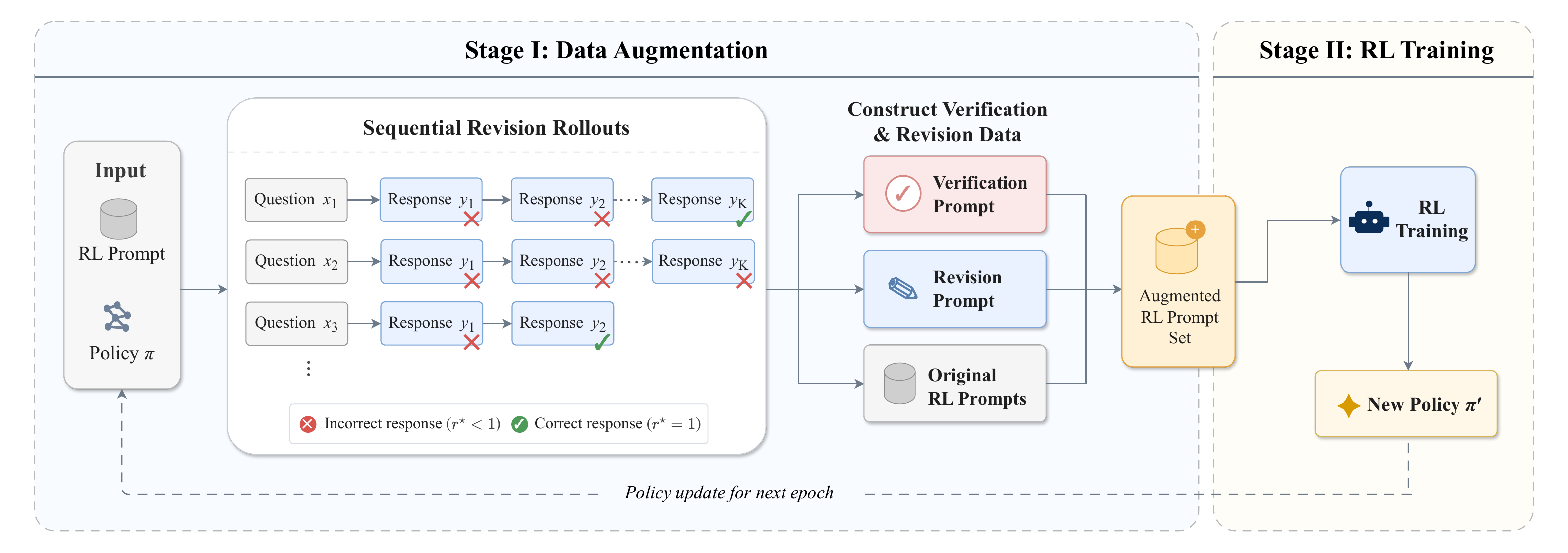}
    \caption{Overview of the proposed RL-based data augmentation framework. In Stage~I, the current policy performs sequential revision; only trajectories that succeed within budget $K$ are retained, and their intermediate states yield revision and verification prompts that are merged with the original RL prompt set. In Stage~II, the policy is updated by single-turn RL on the augmented set, and the new policy regenerates the augmentation in the next epoch.}
    \label{fig:alg_overview}
\end{figure*}

\paragraph{Efficiency.}
Our algorithm is also more efficient than standard multi-turn training. First, augmented samples are generated in an \emph{off-policy} manner. With a well-designed infrastructure, data augmentation can run in parallel with RL training, so the model continues training on the original RL prompts while new augmented data are generated asynchronously, avoiding any interruption to the main training loop. Second, training itself is single-turn. Each gradient step processes one (prompt, response) pair as in ordinary RLVR, with no per-step sequential rollout in the inner loop, which is the dominant cost of multi-turn RL with budget $K$. Appendix~\ref{appendix:efficiency} reports the wall-clock comparison.

\paragraph{Test-time stopping rule.}
At test time the oracle stopping time $\tau$ is almost always unavailable, so we adopt a surrogate stopping time $\hat\tau$ that replaces $\rstar$ with a domain-specific signal.

For mathematical reasoning tasks, we leverage the \emph{Tail Confidence} metric of \citet{fu2025deepthinkconfidence} as the surrogate signal. At each revision step $t$, the token-level confidence at position $i$ of $y_t$ is $C_i = -\frac{1}{k}\sum_{j=1}^{k} \log P_i(j)$, where $P_i(j)$ is the probability of the $j$-th most likely token (a larger $C_i$ indicates a more concentrated, more confident distribution). The response-level score $c_t$ averages $C_i$ over the tail positions $i \in T_{\text{tail}}$ of $y_t$. The stopping time is
$
\hat\tau := K \land \min\!\Big\{ t \ge 3 :{c_t}/{\sum_{j=1}^{t} c_j} > c \Big\}
$
for a predefined threshold $c$. If $\hat\tau = K$ (budget exhausted without trigger), we output $\hat y = y_{t^\ast}$ with $t^\ast = \arg\max_{1 \le t \le K} c_t$, breaking ties uniformly at random when several responses share the same confidence. Selecting the most confident response across all $K$ revisions instead of just $y_K$ improves robustness under small budgets by mitigating noise in confidence estimates at the stopping step.

For coding tasks, we use publicly available test cases as a surrogate reward $\tilde r : \Xcal \times \Ycal \to \{0, 1\}$, with $\tilde r(x, y) = \mathbf{1}(y \text{ passes all public test cases})$. The stopping time is
$
\hat\tau \;:=\; K \land \min\{ t : \tilde r(x, y_t) = 1 \},
$
and we output $\hat y = y_K$ when $\hat\tau = K$.

% !TEX root = main.tex
\section{Experimental Results}
\label{sec:experiment}

% \subsection{Training Setup}

\paragraph{Models and data.}
We train Qwen2.5-7B, Qwen2.5-3B \citep{yang2024qwen2technicalreport}, Qwen3-4B(non-thinking) \citep{yang2025qwen3technicalreport} and DeepSeek-R1-Distill-Qwen-7B \citep{deepseekai2025deepseekr1incentivizingreasoningcapability} directly with RL. Training data is sampled from \href{https://huggingface.co/datasets/Skywork/Skywork-OR1-RL-Data}{\texttt{Skywork/Skywork-OR1-RL-Data}} ($7298$ math + $1015$ coding $= 8313$ examples). We evaluate on MATH500 and AIME 24/25 for math; LiveCodeBench \citep{jain2024livecodebenchholisticcontaminationfree} and CodeContest \citep{doi:10.1126/science.abq1158} for coding; and n\_queens / mini\_sudoku from ReasoningGym \citep{stojanovski2025reasoninggymreasoningenvironments} for out-of-distribution puzzles.

\paragraph{Baselines and implementation.}
We compare against \textbf{RL} (single-response RL, no revision), \textbf{MultiTurn} (multi-turn rollouts without explicit self-verification), and \textbf{PAG} \citep{jiang2025pagmultiturnreinforcedllm} (multi-turn with self-verification via a designed reward). Full details in Appendix~\ref{appendix:implementation}.

\subsection{Main Results}
% Across mathematical reasoning, coding, and out-of-distribution puzzle tasks, our
% REVES method consistently substantially improves models' \emph{sequential revision capability}, leading to test-time sequential revision performance improvement.

\paragraph{Consistent improvements on coding tasks.}
Table~\ref{tab:main_coding} demonstrates the benefits of REVES on coding benchmarks.
Under test-time surrogate stopping criteria, where publicly available test cases are used
to determine when to stop and execution results are provided as feedback for subsequent
revisions (see Appendix~\ref{appendix: setup} for an example), REVES consistently
improves performance across different test-time budgets, demonstrating robustness beyond symbolic rewards. Additional results on DeepSeek-R1-Distill-7B \citep{deepseekai2025deepseekr1incentivizingreasoningcapability} are provided in Appendix~\ref{appendix:strong}.

\begin{table}[!ht]
\centering
\caption{Sequential-revision test-time scaling performance on coding benchmarks (LiveCodeBench (LCB), CodeContest). The best result in each column is shown in bold, and the second-best is underlined. TC-$B$ is sequential-revision (SR) test-time scaling with budget $B$, using execution-based test cases as the verifier and feeding the resulting feedback into the next revision step. }
\small
\setlength{\tabcolsep}{3pt}
\renewcommand{\arraystretch}{1.0}
\fitwidth{
\begin{tabular}{l ccc cc ccc}
\toprule
\multirow{2}{*}{\textbf{Model}}
& \multicolumn{3}{c}{\textbf{LCB (Aug 24 -- Jan 25)}}
& \multicolumn{2}{c}{\textbf{LCB (Jan 25 -- May 25)}}
& \multicolumn{3}{c}{\textbf{CodeContest}} \\
\cmidrule(lr){2-4} \cmidrule(lr){5-6} \cmidrule(lr){7-9}

& \textbf{OneShot}
& \textbf{TC-32}
& \textbf{TC-4}
& \textbf{OneShot}
& \textbf{TC-32}
& \textbf{OneShot}
& \textbf{TC-32}
& \textbf{TC-4} \\
\midrule

Qwen-2.5-7B
& $4.8$ & $6.1$ & $6.1$
& $6.0$ & $18.9$
& $0.4$ & $5.7$ & $5.9$ \\

Qwen-2.5-7B-RL
& $18.0$ & $23.0$ & $20.8$
& $20.0$ & $24.2$
& $\mathbf{7.1}$ & $11.1$ & $7.5$ \\

Qwen-2.5-7B-Multi-turn
& $\underline{19.4}$ & $25.5$ & $21.3$
& $\mathbf{20.7}$ & $\underline{27.5}$
& $\underline{6.9}$ & $14.3$ & $8.1$ \\

Qwen-2.5-7B-PAG
& $\mathbf{19.5}$ & $\underline{25.7}$ & $\underline{23.2}$
& $19.8$ & $26.4$
& $6.5$ & $\underline{15.0}$ & $\underline{9.1}$ \\

\rowcolor{palettePale!45}
Qwen-2.5-7B-REVES
& $18.9$ & $\mathbf{29.5}$ & $\mathbf{23.8}$
& $\underline{20.5}$ & $\mathbf{30.0}$
& $\mathbf{7.1}$ & $\mathbf{16.6}$ & $\mathbf{11.7}$ \\

\midrule

Qwen-2.5-3B
& $5.9$ & $15.8$ & $7.2$
& $7.9$ & $18.3$
& $1.6$ & $6.3$ & $2.4$ \\

Qwen-2.5-3B-RL
& $9.2$ & $12.0$ & $\underline{11.5}$
& $14.1$ & $17.8$
& $3.0$ & $4.2$ & $3.6$ \\

Qwen-2.5-3B-Multi-turn
& $\mathbf{9.9}$ & $\underline{15.7}$ & $\mathbf{12.1}$
& $13.2$ & $\underline{19.2}$
& $\underline{3.2}$ & $\underline{7.9}$ & $4.4$ \\

Qwen-2.5-3B-PAG
& $8.1$ & $14.7$ & $11.2$
& $\underline{13.4}$ & $18.5$
& $\mathbf{3.6}$ & $5.7$ & $\underline{4.9}$ \\

\rowcolor{palettePale!45}
Qwen-2.5-3B-REVES
& $\underline{9.6}$ & $\mathbf{17.7}$ & $11.0$
& $\mathbf{14.3}$ & $\mathbf{21.6}$
& $2.2$ & $\mathbf{8.5}$ & $\mathbf{5.5}$ \\

\midrule

Qwen3-4B (base)
& $22.6$ & $27.8$ & $26.8$
& $24.7$ & $24.7$
& $6.1$ & $10.9$ & $9.7$ \\

Qwen3-4B-RL
& $\underline{32.2}$ & $41.7$ & $42.5$
& $\underline{29.1}$ & $35.4$
& $\underline{16.1}$ & $30.9$ & $\underline{30.9}$ \\

Qwen3-4B-Multi-turn
& $29.2$ & $\underline{49.5}$ & $\underline{44.9}$
& $25.3$ & $\underline{40.9}$
& $10.6$ & $\underline{34.9}$ & $\mathbf{34.2}$ \\

Qwen3-4B-PAG
& $29.8$ & $43.6$ & $40.4$
& $26.7$ & $37.1$
& $13.6$ & $32.4$ & $25.8$ \\

\rowcolor{palettePale!45}
Qwen3-4B-REVES
& $\mathbf{32.7}$ & $\mathbf{50.9}$ & $\mathbf{46.3}$
& $\mathbf{29.7}$ & $\mathbf{42.0}$
& $\mathbf{16.4}$ & $\mathbf{37.9}$ & $\underline{30.9}$ \\

\bottomrule
\end{tabular}}
\label{tab:main_coding}
\end{table}

\paragraph{Improved sequential revision capability for mathematical reasoning.}
Table~\ref{tab:math_main} shows that REVES substantially improves sequential revision performance across all evaluated math benchmarks. Under oracle stopping (the Oracle-32 and Oracle-4 columns, where the rollout halts as soon as a generated answer matches the ground truth, isolating the model's ability to revise an incorrect solution into a correct one), our method achieves consistently large gains. These improvements persist under the SelfConf-4 column, the practical test-time stopping rule based on the model's own confidence, indicating that the learned capability transfers beyond the oracle setting.

\paragraph{Matching state-of-the-art on circle packing.}
We further evaluate REVES on the circle packing benchmark, where the goal is to pack $n$ unit-area circles into a unit square to maximize the sum of their radii. We directly optimize Eq.~\eqref{eq:hazard-decomp} on this task. With a Qwen3-4B, our trained policy matches the previously reported best results from much larger evolutionary search systems built on Gemini-2.0 Pro/Flash and Qwen3-8B (Table~\ref{tab:circle-packing}). Full implementation details are in Appendix~\ref{appendix:circle_packing}.

% Circle-packing results table. Natural width ~313pt < 397pt text width, so it
% renders at a true \small (9pt) with no down-scaling.
\begin{table}[t]
\centering
\small
\setlength{\tabcolsep}{6pt}
\renewcommand{\arraystretch}{1}
\caption{Circle packing benchmark (sum of radii at $n{=}26$, $\uparrow$). REVES on Qwen3-4B matches the best previously reported result, reaching the optimal $2.635983$ with the \emph{smallest} base model and fewer rollouts.}
\label{tab:circle-packing}
\resizebox{0.85\linewidth}{!}{%
\begin{tabular}{l l c}
\toprule
\textbf{Method} & \textbf{Model} & \textbf{Sum of radii, $n{=}26$ ($\uparrow$)} \\
\midrule
AlphaEvolve \citep{novikov2025alphaevolve}      & Gemini-2.0 Pro + Flash & $2.635862$ \\
AlphaEvolve V2 \citep{georgiev2025mathematical} & Gemini-2.0 Pro + Flash & $\mathbf{2.635983}$ \\
ShinkaEvolve \citep{lange2025shinkaevolve}    & Ensemble               & $2.635982$ \\
ThetaEvolve \citep{wang2025thetaevolve}    & R1-Qwen3-8B            & $\mathbf{2.635983}$ \\
TTT-Discover  \citep{yuksekgonul2026learningdiscovertesttime}  & Qwen3-8B               & $\mathbf{2.635983}$ \\
\midrule
\rowcolor{palettePale!45}
Ours (REVES)    & \textbf{Qwen3-4B}      & $\mathbf{2.635983}$ \\
\bottomrule
\end{tabular}}
\end{table}

% OOD puzzle plot (native 2:1 aspect ratio). Width chosen to match the circle-packing
% table's natural footprint (~313pt = ~0.79\linewidth) so the two figures look balanced.
\begin{figure}[t]
\centering
\includegraphics[width=0.75\linewidth]{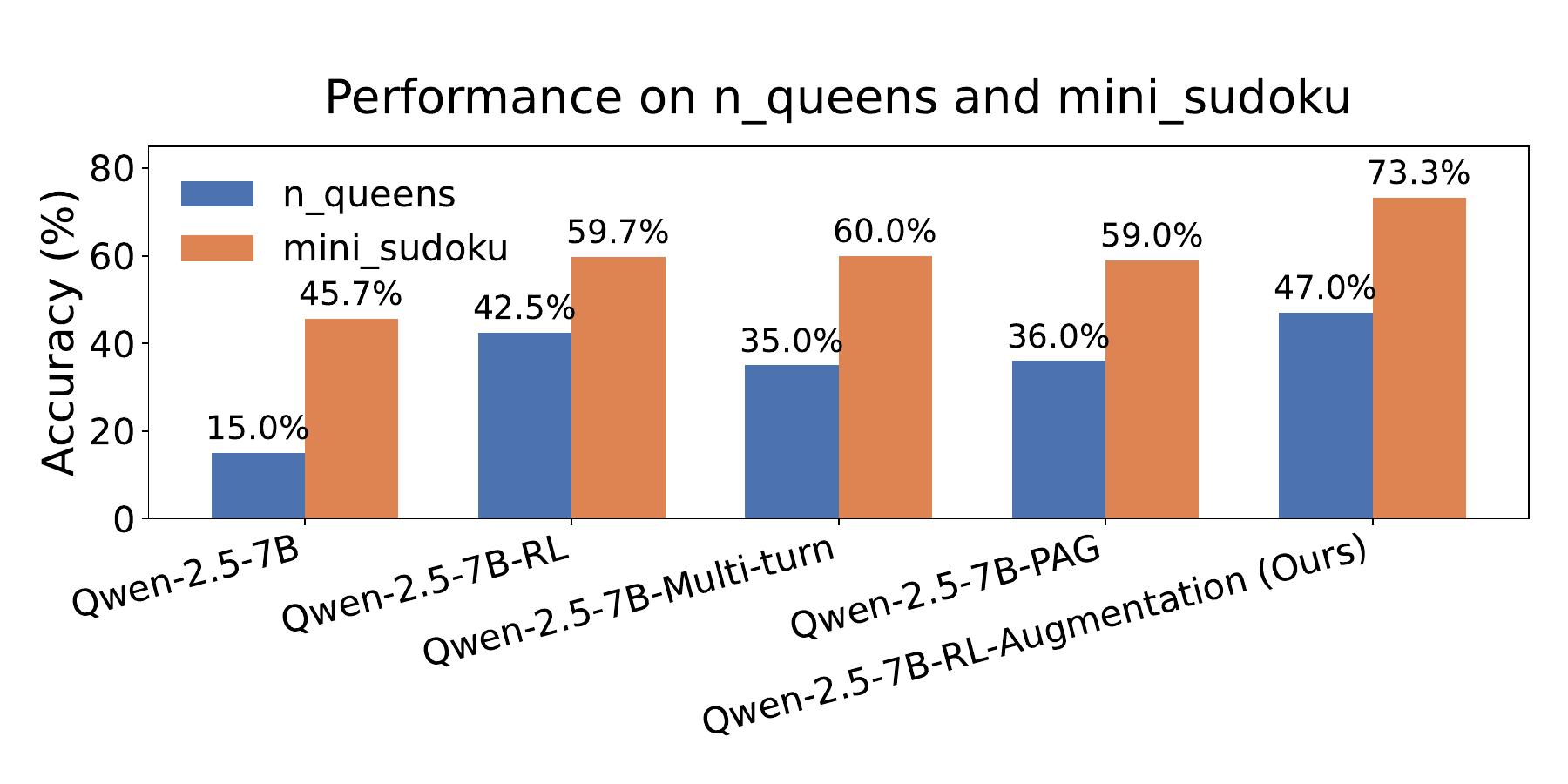}
\caption{Out-of-distribution puzzle benchmarks (n\_queens and mini\_sudoku) under verifier-based stopping. REVES generalizes beyond standard reasoning benchmarks despite using checkpoints trained only on math and code.}
\label{fig:ood-puzzle}
\end{figure}

% \paragraph{Transferable gains under test-time sequential revision.}
%     The improvements generalize across different SeqRev inference settings, including
%     Oracle-32, Oracle-4, and SelfConf-4, without modifying the test-time strategy or budget.
%     This indicates that REVES enhances a \emph{general and test-time-agnostic}
%     sequential revision capability, rather than overfitting to a specific inference rule.

\begin{table*}[htbp]
\centering
\caption{Sequential-revision test-time scaling performance on math benchmarks (AIME24/25, MATH500). \textbf{1-shot} is OneShot; \textbf{O-$B$} (Oracle) and \textbf{SC-$B$} (SelfConf) are sequential-revision (SR) test-time scaling with budget $B$, using the ground-truth answer as the oracle stopping rule and the model's own confidence as a surrogate stopping rule, respectively. $\dagger$ denotes evaluation using the authors' original codebase and protocol. }
\label{tab:math_main}

\small
\setlength{\tabcolsep}{2.5pt}
\renewcommand{\arraystretch}{1.0}

\fitwidth{
\begin{tabular}{l cccc cccc cccc}
\toprule
\multirow{2}{*}{\textbf{Model}}
& \multicolumn{4}{c}{\textbf{AIME 24}}
& \multicolumn{4}{c}{\textbf{AIME 25}}
& \multicolumn{4}{c}{\textbf{MATH 500}} \\
\cmidrule(lr){2-5} \cmidrule(lr){6-9} \cmidrule(lr){10-13}

& \textbf{1-shot}
& \textbf{O-32} & \textbf{O-4} & \textbf{SC-4}
& \textbf{1-shot}
& \textbf{O-32} & \textbf{O-4} & \textbf{SC-4}
& \textbf{1-shot}
& \textbf{O-32} & \textbf{O-4} & \textbf{SC-4} \\
\midrule

Qwen-2.5-7B
& $2.0$ & $14.0$ & $4.8$ & $2.3$
& $0.3$ & $6.3$ & $1.3$ & $0.6$
& $23.8$ & $50.1$ & $29.0$ & $26.1$ \\

Qwen-2.5-7B-RL
& $\mathbf{17.9}$ & $\underline{33.5}$ & $\underline{21.4}$ & $\underline{18.7}$
& $\mathbf{10.4}$ & $22.9$ & $\underline{13.4}$ & $\underline{12.5}$
& $\underline{76.1}$ & $85.9$ & $79.4$ & $\mathbf{78.6}$ \\

Qwen-2.5-7B-Multi-turn
& $14.6$ & $30.3$ & $20.0$ & $13.8^\dagger$
& $7.2$ & $25.3$ & $11.8$ & $9.8^\dagger$
& $75.9$ & $\underline{87.1}$ & $80.1$ & $\underline{78.0}^\dagger$ \\

Qwen-2.5-7B-PAG
& $\underline{15.3}$ & $31.1$ & $19.6$ & $13.0^\dagger$
& $7.5$ & $\underline{26.8}$ & $12.8$ & $9.4^\dagger$
& $\underline{76.1}$ & $86.7$ & $\underline{80.8}$ & $77.2^\dagger$ \\

\rowcolor{palettePale!45}
Qwen-2.5-7B-REVES
& $14.9$ & $\mathbf{45.7}$ & $\mathbf{25.9}$ & $\mathbf{19.6}$
& $\underline{9.3}$ & $\mathbf{40.5}$ & $\mathbf{22.3}$ & $\mathbf{15.1}$
& $\mathbf{76.8}$ & $\mathbf{94.7}$ & $\mathbf{85.6}$ & $76.2$ \\

\midrule

Qwen-2.5-3B
& $1.2$ & $12.9$ & $3.4$ & $1.9$
& $0.5$ & $7.9$ & $1.3$ & $0.6$
& $30.5$ & $67.0$ & $39.3$ & $38.4$ \\

Qwen-2.5-3B-RL
& $\underline{6.4}$ & $16.9$ & $8.1$ & $7.4$
& $1.8$ & $9.0$ & $3.9$ & $2.4$
& $\underline{64.9}$ & $80.5$ & $69.6$ & $65.6$ \\

Qwen-2.5-3B-Multi-turn
& $\mathbf{7.6}$ & $\underline{20.5}$ & $\underline{12.0}$ & $\underline{8.0}^\dagger$
& $2.6$ & $\underline{12.7}$ & $\underline{5.2}$ & $2.3^\dagger$
& $\mathbf{65.3}$ & $\underline{81.2}$ & $\underline{72.7}$ & $\mathbf{67.0}^\dagger$ \\

Qwen-2.5-3B-PAG
& $6.3$ & $14.8$ & $7.9$ & $5.3^\dagger$
& $\mathbf{3.4}$ & $11.4$ & $4.7$ & $\underline{4.1}^\dagger$
& $64.0$ & $78.8$ & $69.9$ & $65.6^\dagger$ \\

\rowcolor{palettePale!45}
Qwen-2.5-3B-REVES
& $6.2$ & $\mathbf{33.7}$ & $\mathbf{14.7}$ & $\mathbf{8.2}$
& $\underline{2.8}$ & $\mathbf{24.4}$ & $\mathbf{7.2}$ & $\mathbf{4.2}$
& $64.5$ & $\mathbf{89.9}$ & $\mathbf{77.8}$ & $\underline{66.0}$ \\

\bottomrule
\end{tabular}
}
\end{table*}

\paragraph{Generalization to out-of-distribution puzzle tasks.}
As illustrated in Figure~\ref{fig:ood-puzzle}, we evaluate puzzle tasks using checkpoints trained only on math and code data, without any puzzle-specific training. Despite this distribution shift, REVES substantially improves performance on out-of-distribution puzzle benchmarks, indicating strong cross-task generalization.

\paragraph{$\JSR$-trained policies also improve other TTS algorithms.}
    Table~\ref{tab:aux_tts} shows that models trained with REVES consistently improve the performance
    of a wide range of test-time algorithms that rely on sequential revision, including
    MCTS \citep{inoue2025widerdeeperscalingllm}, AB-MCTS variants \citep{inoue2025widerdeeperscalingllm} and Mind Evolution \citep{lee2025evolvingdeeperllmthinking}.
    This suggests that sequential revision capability serves as a fundamental backbone that
    raises the performance ceiling of diverse test-time inference strategies.

\begin{table*}[!ht]
\centering
\caption{Comparison of model performance on different test-time strategies. Best results in each column are in bold; second-best are underlined. \textbf{1-shot} is OneShot; \textbf{MindEvo} is Mind Evolution; \textbf{ABM-A($\beta$)}, \textbf{ABM-A(G)}, and \textbf{ABM-M} are the AB-MCTS-A (Beta), AB-MCTS-A (Gaussian), and AB-MCTS-M variants.}
\label{tab:aux_tts}
\small
\setlength{\tabcolsep}{2pt}
\fitwidth{
\begin{tabular}{l cccccccc}
\toprule
\textbf{Model}
 & \textbf{1-shot} & \textbf{SeqRev} & \textbf{BoN} & \textbf{MCTS}
 & \textbf{MindEvo} & \textbf{ABM-A($\beta$)} & \textbf{ABM-A(G)} & \textbf{ABM-M} \\
\midrule
Qwen-2.5-7B-RL
 & $18.0$
 & $23.0$
 & $27.8$
 & $22.7$
 & $25.7$
 & $26.0$
 & $27.3$
 & $26.2$ \\
Qwen-2.5-7B-Multi-turn
 & \underline{$19.4$}
 & $25.5$
 & $28.7$
 & $26.1$
 & $27.0$
 & $28.6$
 & \underline{$29.0$}
 & $27.8$ \\
Qwen-2.5-7B-PAG
 & $\mathbf{19.5}$
 & \underline{$25.7$}
 & $\mathbf{30.4}$
 & \underline{$26.4$}
 & \underline{$27.7$}
 & \underline{$29.7$}
 & $28.1$
 & \underline{$28.7$} \\
\rowcolor{palettePale!45}
Qwen-2.5-7B-REVES
 & $18.9$
 & $\mathbf{29.5}$
 & \underline{$29.3$}
 & $\mathbf{28.9}$
 & $\mathbf{29.0}$
 & $\mathbf{31.0}$
 & $\mathbf{30.0}$
 & $\mathbf{30.4}$ \\
\bottomrule
\end{tabular}
}
\end{table*}

\begin{tightbox}{paletteNavy}{Takeaway: REVES delivers consistent, transferable test-time gains.}
REVES improves sequential revision on code and math, matches state-of-the-art on frontier open problems (circle packing) with a far smaller model, and generalizes to out-of-distribution puzzles despite training only on math and code. As predicted by Theorem~\ref{thm:whysrinformal}, the gains further transfer to other revision-using TTS algorithms.
\end{tightbox}

\subsection{Ablation Study}
\label{sec:ablation}
\paragraph{The role of revision and verification prompts.}
As shown in Table~\ref{tab:ablation-verification}, we investigate the roles of revision and verification prompts in sequential revision training.
Qwen-2.5-7B-RevisionOnly augments RL-generated data using only revision prompts, while
Qwen-2.5-7B-VerificationOnly augments data using only verification prompts.
Qwen-2.5-7B-REVES combines both prompts during data augmentation.
Results under Oracle-4 indicate that revision prompts primarily improve \emph{sequential revision capability}, whereas verification prompts do not directly enhance revision but instead train the model to assess solution validity.

To isolate the effect of verification prompts, we evaluate confidence calibration using Area Under the Receiver Operating Characteristic curve (AUROC) on AIME25.
For each problem, we sample multiple model responses, label them as correct or incorrect based on the final answer, and compute the probability that a randomly sampled correct response receives a higher \texttt{TailConfidence} score than a randomly sampled incorrect one.
The final AUROC is obtained by averaging over problems with sufficient correct samples.
Under this evaluation, Qwen-2.5-7B-REVES achieves a higher AUROC (74.1\%) than RevisionOnly (72.1\%), demonstrating that verification prompts improve confidence calibration.
Thus, combining revision and verification prompts yields consistent gains by jointly enhancing revision capability and confidence estimation.

\begin{table}[t]
\centering
\caption{Effect of revision and verification prompts on AIME25. \textbf{Left:} 1-shot (OneShot), O-4 (Oracle-4), and SC-4 (SelfConf-4) accuracy across the four configurations: \emph{RL} (no augmentation), \emph{VerificationOnly} (verification prompts only), \emph{RevisionOnly} (revision prompts only), and \emph{REVES} (both, our method). \textbf{Right:} \texttt{TailConfidence} distributions for RevisionOnly and REVES; REVES produces sharper, better-separated correct/incorrect distributions, indicating improved confidence calibration.}
\label{tab:ablation-verification}
\begin{minipage}[c]{0.45\linewidth}
\centering
\resizebox{\linewidth}{!}{%
\begin{tabular}{l ccc}
\toprule
\textbf{Model} & \textbf{1-shot} & \textbf{O-4} & \textbf{SC-4} \\
\midrule
Qwen-2.5-7B-RL                 & $\mathbf{10.4}$ & $13.4$ & $\underline{12.5}$ \\
Qwen-2.5-7B-VerificationOnly   & $6.3$  & $10.3$ & $7.1$  \\
Qwen-2.5-7B-RevisionOnly       & $7.9$  & $\underline{16.5}$ & $8.3$  \\
Qwen-2.5-7B-REVES    & $\underline{9.5}$  & $\mathbf{21.2}$ & $\mathbf{15.8}$ \\
\bottomrule
\end{tabular}}
\end{minipage}%
\hfill
\begin{minipage}[c]{0.53\linewidth}
\centering
\includegraphics[width=\linewidth]{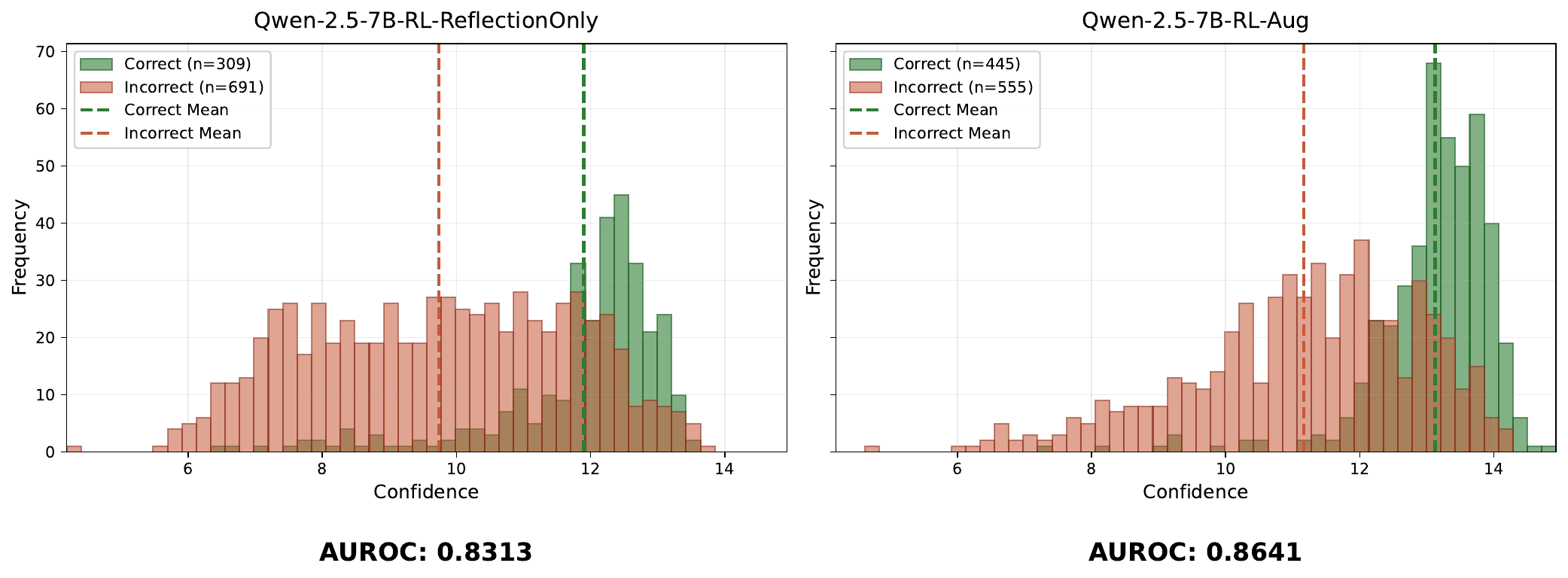}
\end{minipage}
\end{table}
% \vspace{-20pt}
\paragraph{Necessity of continual data augmentation.}
As shown in Figure~\ref{fig:ablation-aime} (right), augmenting data only at the initial epoch leads to substantially worse performance than continuous augmentation. This is because once the model stops making earlier mistakes, stale augmented data no longer provides informative learning signals, whereas regenerating data each epoch keeps the supervision aligned with the model's current failure modes.

% \begin{figure}[t]
%   \centering
%   \begin{minipage}[t]{0.48\columnwidth}
%     \centering
%     \includegraphics[width=\linewidth]{data_aug_ablation_budget_vs_acc_with1_corrected.pdf}
%   \end{minipage}
%   \hfill
%   \begin{minipage}[t]{0.48\columnwidth}
%     \centering
%     \includegraphics[width=\linewidth]{math500_budget_vs_acc.pdf}
%   \end{minipage}
%   \caption{Comparison between AIME 24 and AIME 25.}
% \end{figure}

\paragraph{Effect of augmentation budget.}
We study the effect of the augmentation budget on training performance. As shown in Figure~\ref{fig:ablation-aime} (left), increasing the augmentation budget consistently improves performance, indicating that our method exhibits strong scaling behavior with respect to the augmentation budget. At the same time, allowing a larger augmentation budget $K$ naturally incurs higher computational cost, reflecting a trade-off between performance and compute in data augmentation.

% AIME24 ablations -- both panels share the same aspect ratio so they align.
\begin{figure}[tb]
    \centering
    \begin{minipage}[t]{0.4\linewidth}\centering
        \includegraphics[width=\linewidth]{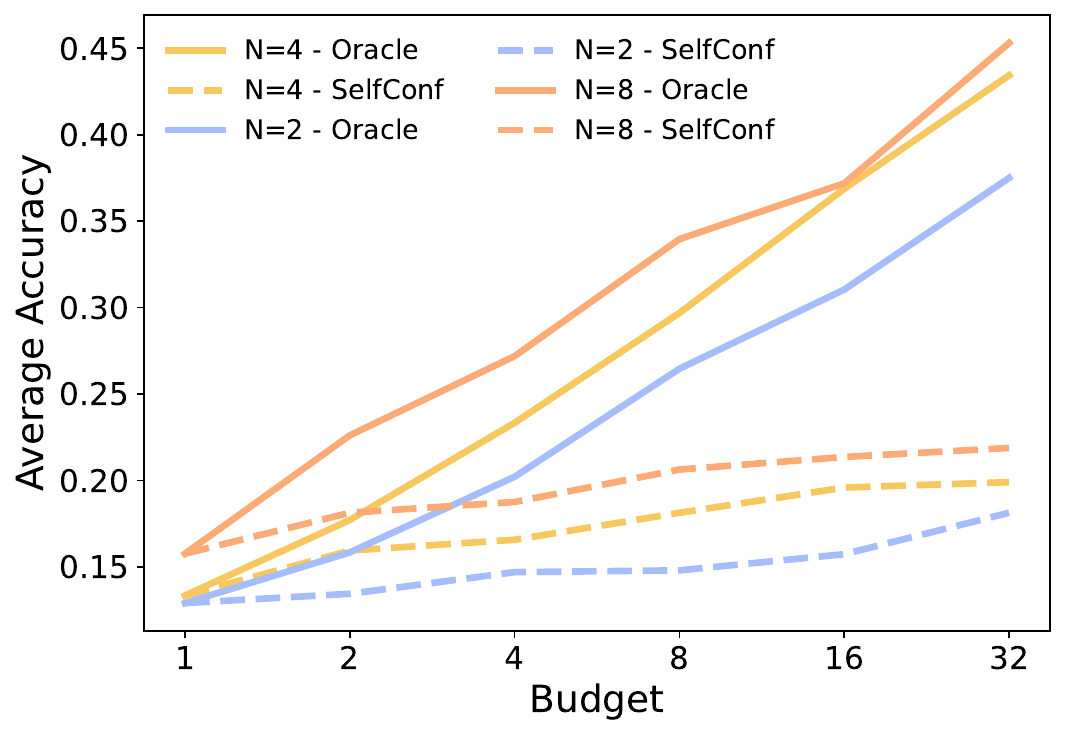}
    \end{minipage}
    \begin{minipage}[t]{0.4\linewidth}\centering
        \includegraphics[width=\linewidth]{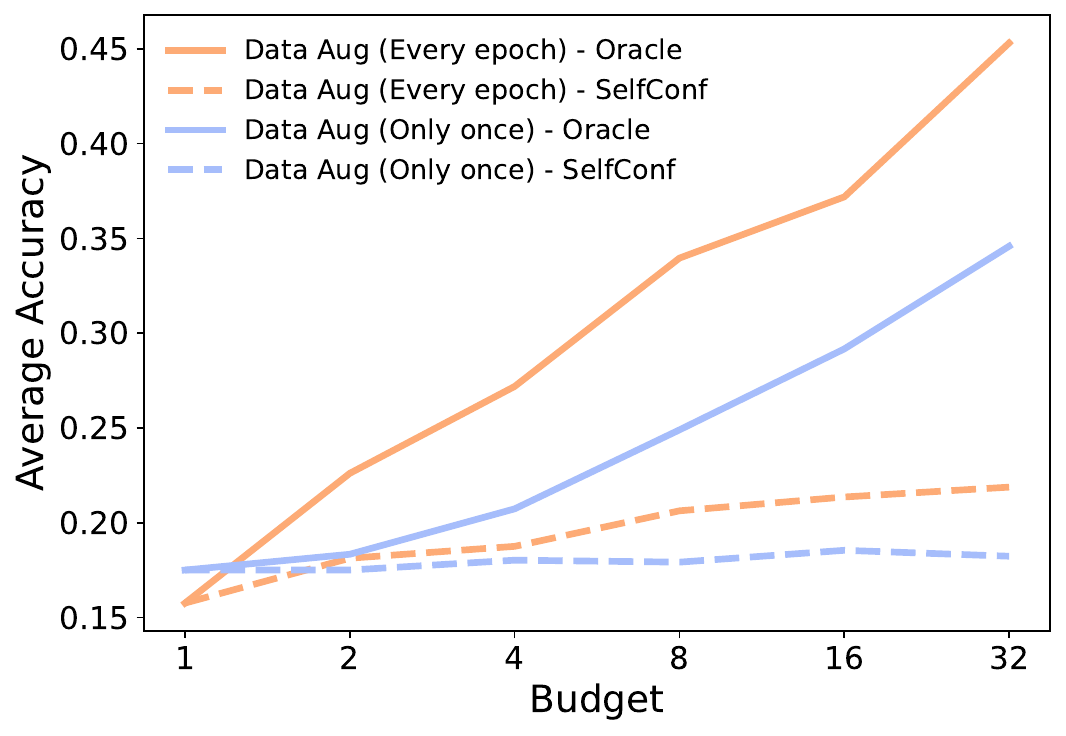}
    \end{minipage}
    \caption{Ablations on data augmentation (AIME24). \emph{Left:} effect of augmentation budget $K$ (the maximum trajectory length allowed during augmentation); \emph{right:} continual versus one-shot augmentation.}
    \label{fig:ablation-aime}
\end{figure}

% MATH500 ablations -- both panels share the same aspect ratio so they align.
\begin{figure}[tb]
    \centering
    \begin{minipage}[t]{0.4\linewidth}\centering
        \includegraphics[width=\linewidth]{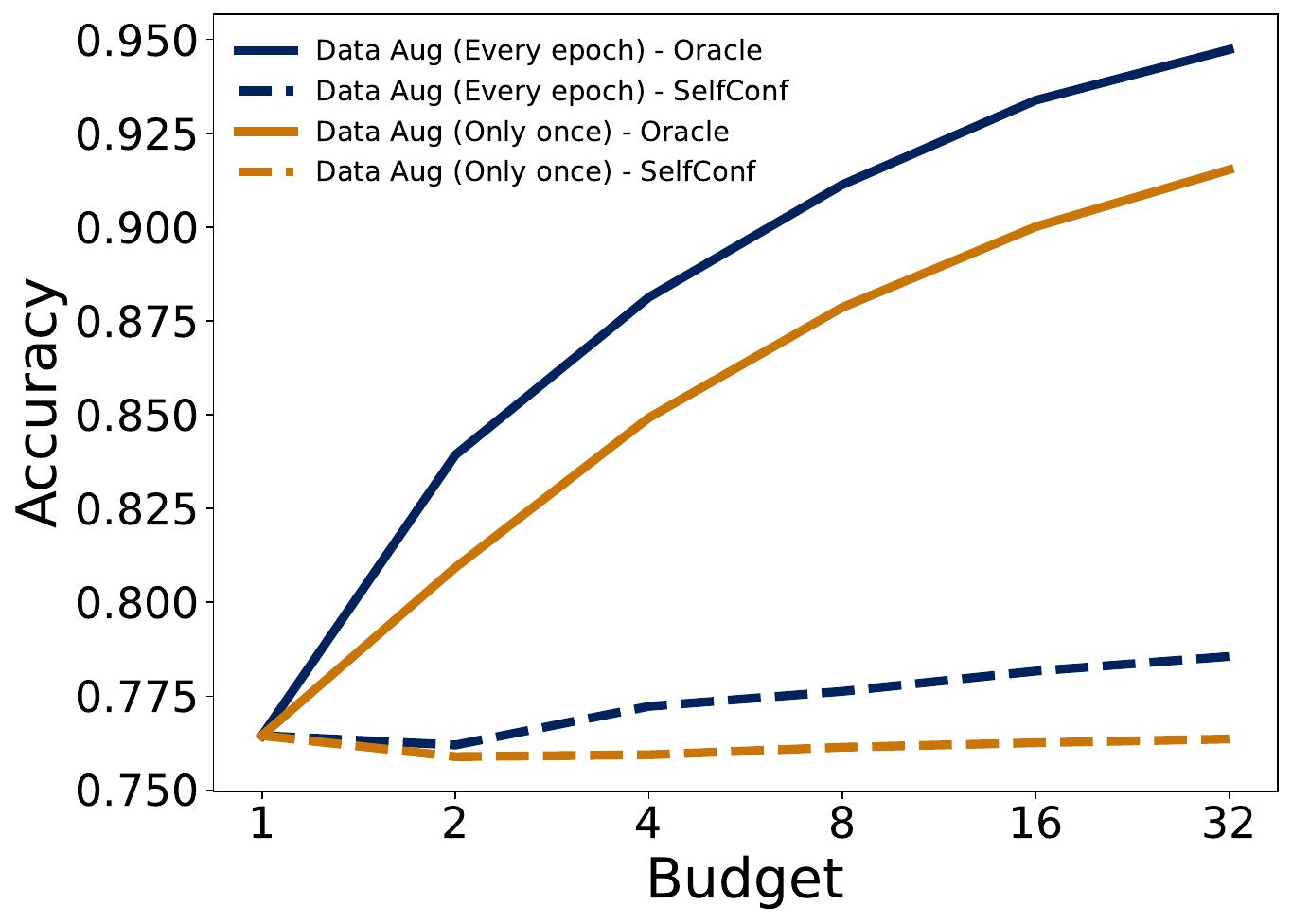}
    \end{minipage}
    \begin{minipage}[t]{0.4\linewidth}\centering
        \includegraphics[width=\linewidth]{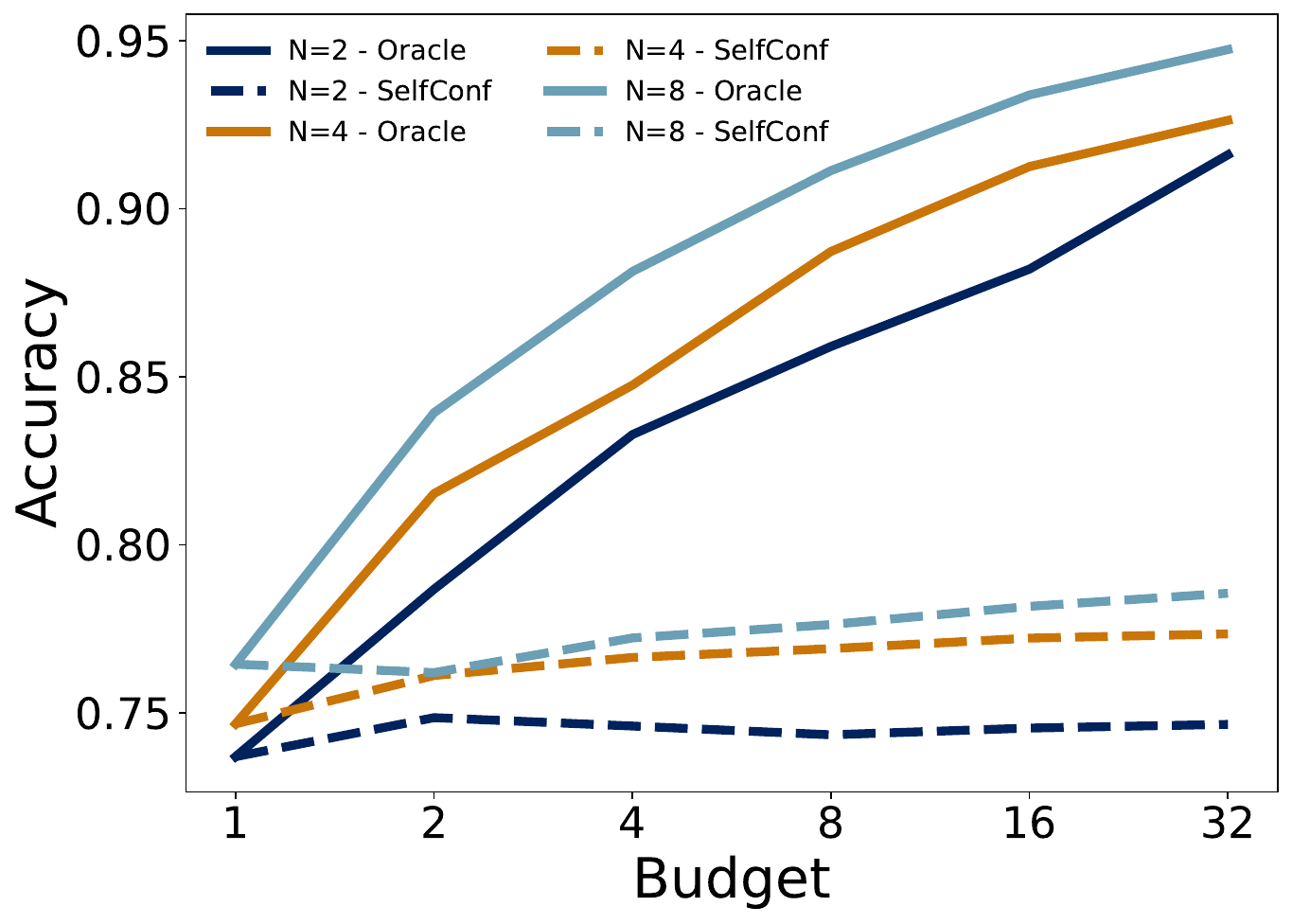}
    \end{minipage}
    \caption{Ablations on data augmentation (MATH500). \emph{Left:} effect of augmentation frequency (continual versus one-shot); \emph{right:} effect of allowed revision budget $K$. }
    \label{fig:ablation-math}
\end{figure}

% Flush all pending floats before the bibliography so figures cannot drift into
% the References list.
% \FloatBarrier

\section{Conclusion}

We show that standard single-shot RL objectives are misaligned with test-time sequential revision, and propose \textbf{REVES}, which pairs RL with continual, mistake-driven data augmentation so the model learns both to revise incorrect solutions and to identify errors. Augmentation must be continual, since static data becomes uninformative once mistakes disappear. Across math, coding, and out-of-distribution tasks, REVES improves test-time sequential-revision performance over the single-shot RL and multi-turn baselines under fixed budgets, and, as predicted by our transfer analysis, these gains further carry over to other revision-using test-time algorithms. On the circle packing benchmark, a Qwen3-4B trained with REVES matches the previously reported state-of-the-art achieved by much larger evolutionary search systems, using the smallest base model and far fewer rollouts. d

% \newpage
\bibliographystyle{colm2026_conference}
\bibliography{references}

%%%%%%%%%%%%%%%%%%%%%%%%%%%%%%%%%%%%%%%%%%%%%%%%%%%%%%%%%%%%%%%%%%%%%%%%%%%%%%%%%
\appendix
% !TEX root = main.tex
% \input{related_work}
\section{Sequential Revision Protocol}
\label{appendix: setup}

In this section, we elaborate on the test-time protocol described in Protocol~\ref{alg:seqrev_setup} and illustrate its practical usage with several concrete examples, including a mathematical reasoning task in Appendix~\ref{appendix: math_exp} and a coding task in Appendix~\ref{appendix: code_exp}.

\floatname{algorithm}{Protocol}
\begin{algorithm}[ht]
  \caption{Sequential Revision Protocol}
  \label{alg:seqrev_setup}
  \begin{algorithmic}[1]
    \STATE \textbf{Input:} revision budget $K$, question $x$, model $\pi$
    \STATE Sample initial response $y_1 \sim \pi(\cdot \mid x)$
    
    \FOR{$i = 2$ {\bfseries to} $K$}
        \STATE Generate revised response $y_i \sim \pi(\cdot \mid x, y_{i-1}, \texttt{Feedback}(y_{i-1}))$
        \IF{$r^\star(x, y_{i}) = 1$}
            \STATE \# Early stopping
            \STATE Set $\hat{y} = y_i$
            \STATE \textbf{break for loop}
        \ENDIF
    \ENDFOR
    
    \IF{No early stopping}
        \STATE $\hat{y} = y_K$
    \ENDIF
    
  \end{algorithmic}
\end{algorithm}

\subsection{Example Math Task}
\label{appendix: math_exp}

In mathematical tasks, the feedback consists solely of the sentence \textcolor{red}{“The response is incorrect.”} (marked in red in the figure), without offering further diagnostic information.

\begin{inputbox}{Input (Turn 1)}
Solve the following math problem. Show brief reasoning and provide the final answer.

Problem:
A $90^\circ$ rotation around $-4 - 5i$ in the clockwise direction is applied to $3$ (as a complex number).  
What is the resulting complex number?
\end{inputbox}

\begin{outputbox}{LLM Output (Turn 1)}
To solve this problem, we need to perform a $90^\circ$ clockwise rotation around the complex number $-4 - 5i$. 

First, it's useful to understand that a $90^\circ$ clockwise rotation of a complex number $z$ around another complex number $a$ is equivalent to the following transformation:
\[
z' = a + (a - z) \cdot i
\]
However, since we are rotating around $-4 - 5i$, let's denote this center of rotation as $a = -4 - 5i$. Our complex number $z = 3$ needs to be rotated.

1. Find the difference between $z$ and $a$:
\[
z - a = 3 - (-4 - 5i) = 3 + 4 + 5i = 7 + 5i
\]

2. Perform the $90^\circ$ clockwise rotation by multiplying this difference by $i$ and adding it back to $a$:
\[
z' = a + (z - a) \cdot i = (-4 - 5i) + (7 + 5i) \cdot i
\]

Next, simplify the expression:
\[
(7 + 5i) \cdot i = 7i + 5i^2
\]
Since \(i^2 = -1\):
\[
7i + 5(-1) = 7i - 5
\]

$$\cdots$$
Therefore, the final answer is $\boxed{-9 + 2i}$.
\end{outputbox}

\begin{inputbox}{Input (Turn 2)}
Problem:
A $90^\circ$ rotation around $-4 - 5i$ in the clockwise direction is applied to $3$ (as a complex number).  What is the resulting complex number?

Previous Attempt:
To solve this problem, we need to perform a $90^\circ$ clockwise rotation around the complex number $-4 - 5i$. 

First, it's useful to understand that a $90^\circ$ clockwise rotation of a complex number $z$ around another complex number $a$ is equivalent to the following transformation:
\[ z' = a + (a - z) \cdot i \]
However, since we are rotating around $-4 - 5i$, let's denote this center of rotation as $a = -4 - 5i$. Our complex number $z = 3$ needs to be rotated.

1. Find the difference between $z$ and $a$:
\[ z - a = 3 - (-4 - 5i) = 3 + 4 + 5i = 7 + 5i \]

2. Perform the $90^\circ$ clockwise rotation by multiplying this difference by $i$ and adding it back to $a$:
\[ z' = a + (z - a) \cdot i = (-4 - 5i) + (7 + 5i) \cdot i \]

Next, simplify the expression:
\[ (7 + 5i) \cdot i = 7i + 5i^2 \]
Since \(i^2 = -1\):
\[ 7i + 5(-1) = 7i - 5 \]
\[\cdots \]

So, the resulting complex number after the rotation is \(-9 + 2i\).

Therefore, the final answer is \(\boxed{-9 + 2i}\).

Verifiable Feedback:
{\color{red}{The response is incorrect.}}

Now produce a corrected solution. Keep reasoning concise.
\end{inputbox}

\begin{outputbox}{LLM Output (Turn 2)}
To correctly rotate a complex number \(3\) around \(-4 - 5i\) by \(90^\circ\) clockwise, we need to use the appropriate transformation for rotations in the complex plane.

Given a complex number \(z = 3\), and the center of rotation \(a = -4 - 5i\), a \(90^\circ\) clockwise rotation of \(z\) around \(a\) can be represented as:
\[ z' = a + (a - z) \cdot (-i) \]

Let's go through the steps in detail:

1. \textbf{Calculate \(a - z\):}
\[ a - z = (-4 - 5i) - 3 = -4 - 5i - 3 = -7 - 5i \]

2. \textbf{Multiply by \(-i\):}
\[ (a - z) \cdot (-i) = (-7 - 5i) \cdot (-i) = (-7)(-i) + (-5i)(-i) = 7i + 5i^2 \]
Since \(i^2 = -1\):
\[ 7i + 5(-1) = 7i - 5 \]

3. \textbf{Add this result back to \(a\):}
\[ z' = a + (7i - 5) = (-4 - 5i) + (-5 + 7i) \]
\[ z' = -4 - 5i - 5 + 7i \]
Combine the real and imaginary parts:
\[ z' = (-4 - 5) + (-5i + 7i) = -9 + 2i \]

Thus, after performing the \(90^\circ\) clockwise rotation, the resulting complex number is \(\boxed{-9 + 2i}\).
\end{outputbox}

\begin{inputbox}{Input (Turn 3)}
Problem:
A $90^\circ$ rotation around $-4 - 5i$ in the clockwise direction is applied to $3$ (as a complex number).  What is the resulting complex number?

Previous Attempt:
To correctly rotate a complex number \(3\) around \(-4 - 5i\) by \(90^\circ\) clockwise, we need to use the appropriate transformation for rotations in the complex plane.

Given a complex number \(z = 3\), and the center of rotation \(a = -4 - 5i\), a \(90^\circ\) clockwise rotation of \(z\) around \(a\) can be represented as:
\[ z' = a + (a - z) \cdot (-i) \]

Let's go through the steps in detail:

1. \textbf{Calculate \(a - z\):}
\[ a - z = (-4 - 5i) - 3 = -4 - 5i - 3 = -7 - 5i \]

2. \textbf{Multiply by \(-i\):}
\[ (a - z) \cdot (-i) = (-7 - 5i) \cdot (-i) = (-7)(-i) + (-5i)(-i) = 7i + 5i^2 \]
Since \(i^2 = -1\):
\[ 7i + 5(-1) = 7i - 5 \]

3. \textbf{Add this result back to \(a\):}
\[ z' = a + (7i - 5) = (-4 - 5i) + (-5 + 7i) \]
\[ z' = -4 - 5i - 5 + 7i \]
Combine the real and imaginary parts:
\[ z' = (-4 - 5) + (-5i + 7i) = -9 + 2i \]

Thus, after performing the \(90^\circ\) clockwise rotation, the resulting complex number is \(\boxed{-9 + 2i}\).

Verifiable Feedback:
{\color{red}{The response is incorrect.}}

Now produce a corrected solution. Keep reasoning concise.
\end{inputbox}

\begin{outputbox}{LLM Output (Turn 3)}
To find the result after a \(90^\circ\) clockwise rotation around \(-4 - 5i\), we can use the following transformation steps. We start by considering a complex number \(z = 3\) and the center of rotation \(a = -4 - 5i\):

1. \textbf{Move \(z\) to the origin relative to \(a\):}
\[
z - a = 3 - (-4 - 5i) = 3 + 4 + 5i = 7 + 5i
\]

2. \textbf{Rotate by \(90^\circ\) clockwise:}
A \(90^\circ\) clockwise rotation of a complex number \(w\) is equivalent to multiplying \(w\) by \(-i\). So, we need to multiply \(7 + 5i\) by \(-i\):
\[
(7 + 5i) \times (-i) = 7 \times (-i) + 5i \times (-i) = -7i - 5i^2
\]
Since \(i^2 = -1\):
\[
-7i - 5(-1) = -7i + 5 = 5 - 7i
\]

3. \textbf{Translate back by adding \(a\):}
Adding \(a\) to the resultant, we get:
\[
z' = (5 - 7i) + (-4 - 5i) = 5 - 4 - 7i - 5i = 1 - 12i
\]

Thus, the resulting complex number after the \(90^\circ\) clockwise rotation is \(\boxed{1 - 12i}\).
\end{outputbox}

\subsection{Example Coding Task}
\label{appendix: code_exp}
At test time for coding tasks, we follow an iterative execution-based evaluation protocol. The model first generates a candidate program, which is then executed against the publicly available test cases provided by the problem. If the program fails any test, the resulting error messages or failing test feedback are collected and appended to the context for the next generation step, allowing the model to revise its solution based on execution feedback. Note that passing the public test cases does not guarantee correctness: a solution is considered correct only if it passes all test cases, including both public and private ones.

\begin{inputbox}{Input (Turn 1)}

\noindent
You are an expert Python programmer. You will be given a problem specification and must produce a correct Python program that passes the tests.

\medskip
\noindent
\textbf{Question:}  
You are given a permutation $P=(P_1,P_2,\dots,P_N)$ of $(1,2,\dots,N)$.
You want to satisfy $P_i=i$ for all $i=1,2,\dots,N$ by performing the following operation zero or more times:

\begin{itemize}
  \item Choose an integer $k$ such that $1 \leq k \leq N$. If $k \geq 2$, sort the $1$-st through $(k-1)$-th terms of $P$ in ascending order. Then, if $k \leq N-1$, sort the $(k+1)$-th through $N$-th terms of $P$ in ascending order.
\end{itemize}

It can be proved that under the constraints of this problem, it is possible to satisfy $P_i=i$ for all $i=1,2,\dots,N$ with a finite number of operations for any $P$. Find the minimum number of operations required.
You have $T$ test cases to solve.

\medskip
\noindent
\textbf{Input}
\dots

\medskip
\noindent
\textbf{Output}

Print $T$ lines. The $i$-th line should contain the answer for the $i$-th test case.

\medskip
\noindent
\textbf{Constraints}
\dots

\medskip
\noindent
\textbf{Sample Input 1}
\begin{itemize}
  \item \texttt{3}
  \item \texttt{5}
  \item \dots
\end{itemize}

\medskip
\noindent
\textbf{Sample Output 1}
\begin{itemize}
  \item \texttt{1}
  \item \dots
\end{itemize}

\medskip
\noindent
\textbf{Format:} Read inputs from stdin and write the answer to stdout. Enclose your final code in the following box.

\medskip
\noindent
\fbox{%
\begin{minipage}{0.96\linewidth}
\ttfamily\small
\# YOUR CODE HERE
\end{minipage}
}

\medskip
\noindent
\textbf{Answer:} Provide \emph{only} the final Python code.

\end{inputbox}

\begin{outputbox}{LLM Output (Turn 1)}
\noindent
To solve this problem, we aim to determine the minimum number of operations required to transform a given permutation
$P$ into the identity permutation $(1,2,\dots,N)$.
The key insight is that each operation can be viewed as a global reordering step that partially resolves disorder in the permutation.
Intuitively, such operations progressively eliminate inconsistencies between positions and values.

\dots

\medskip
\noindent
\fbox{%
\begin{minipage}{0.96\linewidth}
\ttfamily\small
def min\_operations(N, P):\\
\hspace*{1.5em}count = 0\\
\hspace*{1.5em}for i in range(N):\\
\hspace*{3em}if P[i] != i + 1:\\
\hspace*{4.5em}count += 1\\
\hspace*{1.5em}return count\\
\\
\dots\\
\\
for \_ in range(T):\\
\hspace*{1.5em}N = int(input())\\
\hspace*{1.5em}P = list(map(int, input().split()))\\
\hspace*{1.5em}print(min\_operations(N, P))
\end{minipage}
}
\end{outputbox}

\begin{inputbox}{Input (Turn 2)}

\noindent
You are an expert Python programmer. Given a problem specification, produce a correct Python program that passes all tests.

\medskip
\noindent
\textbf{Question:}
Given a permutation $P=(P_1,\dots,P_N)$ of $(1,\dots,N)$, the goal is to transform $P$ into the identity permutation \dots

\medskip
\noindent
\textbf{Input / Output.}
\dots

\bigskip
\noindent
\textbf{Previous Codes}

\medskip
\noindent
\textbf{code\_v1}

\medskip
\noindent
\fbox{%
\begin{minipage}{0.96\linewidth}
\ttfamily\small
\# Your task is to find the minimum number of operations required to sort the permutation P.\\
\\
def min\_operations(N, P):\\
\hspace*{1.5em}operations = 0\\
\hspace*{1.5em}for i in range(N):\\
\hspace*{3em}while P[i] != i + 1:\\
\hspace*{6em} \dots \\
\hspace*{1.5em}N = case[0]\\
\hspace*{1.5em}P = case[1:]\\
\hspace*{1.5em}print(min\_operations(N, P))
\end{minipage}
}

\medskip
\noindent
\textbf{Issues with code\_v1}

\medskip
For the input
\[
\begin{aligned}
&3\\
&5\\
&2\ 1\ 3\ 5\ 4\\
&3\\
&1\ 2\ 3\\
&7\\
&3\ 2\ 1\ 7\ 5\ 6\ 4
\end{aligned}
\]
the expected output is
\[
\begin{aligned}
&1\\
&0\\
&2
\end{aligned}
\]
but \texttt{code\_v1} results in a runtime error.
The failure is primarily due to incorrect handling of index-based updates and repeated linear searches (\texttt{P.index}),
which lead to inconsistent state updates and excessive runtime overhead.

\end{inputbox}

\begin{outputbox}{LLM Output (Turn 2)}
\noindent
To solve this problem, we aim to determine the minimum number of operations required to transform a given permutation
$P$ into the identity permutation $(1,2,\dots,N)$.
The key insight is that each operation can be viewed as a global reordering step that partially resolves disorder in the permutation.
Intuitively, such operations progressively eliminate inconsistencies between positions and values.

\dots

\medskip
\noindent
\fbox{%
\begin{minipage}{0.96\linewidth}
\ttfamily\small
def min\_operations(N, P):\\
\hspace*{1.5em}count = 0\\
\hspace*{1.5em}for i in range(N):\\
\hspace*{3em}if P[i] != i + 1:\\
\hspace*{4.5em}count += 1\\
\hspace*{1.5em}return count\\
\\
\dots\\
\\
for \_ in range(T):\\
\hspace*{1.5em}N = int(input())\\
\hspace*{1.5em}P = list(map(int, input().split()))\\
\hspace*{1.5em}print(min\_operations(N, P))
\end{minipage}
}
\end{outputbox}

\section{Limitation and Future Direction}
\label{appendix:limitation}

\paragraph{Limitation.}
REVES requires a ground-truth verifier during training. Math, coding, puzzle, and circle-packing benchmarks all admit deterministic verifiers, so REVES applies directly. For tasks without such verifiers (free-form QA, creative writing, RLHF-style preference learning), REVES does not directly extend; the verification prompts trained alongside revision data can be repurposed for self-verification at test time, but their reliability depends on the policy's own calibration.

\paragraph{Future direction.}
The verifier requirement points to a natural class of problems where REVES is especially well suited: scientific tasks whose objectives admit a clear, automated reward. Examples include solving partial differential equations (PDEs), where residual norms or analytic checks serve as verifiers; symbolic regression, where fit and complexity supply a numeric score; theorem proving, where a proof checker provides binary feedback; and discovery-style optimization tasks like circle packing (Appendix~\ref{appendix:circle_packing}), where the verifier is the geometric solver itself. In all of these settings, the cost of producing the verifier signal is offset by an essentially unlimited supply of training problems and a precise notion of correctness, both of which line up well with REVES's reliance on per-state one-step recovery as the gradient signal. Scaling REVES to such scientific-discovery domains is a natural next step.

% \newpage
\section{Proofs}
\label{appendix: proofs}

In this section we collect the proofs of the theoretical results stated in the main text. We first prove the objective-mismatch result (Theorem~\ref{thm:mismatch}), showing that optimizing single-shot pass@1 does not optimize the sequential-revision objective $\JSR$. We then prove the transfer guarantee (Theorem~\ref{thm:whysrinformal}), which establishes that improving SR recovery also improves every revision-using TTS algorithm that covers SR's revision inputs. Finally, we prove the hazard decomposition (Lemma~\ref{lem:hazard}) that exposes the per-state one-step recovery structure underlying REVES.

\subsection{Proof of Theorem~\ref{thm:mismatch}}

\begin{theorem}[Objective mismatch]
For any revision length $K \geq 2$, there exist policies $\pi_1, \pi_2$ and a problem distribution $\Xcal$ such that
\[
\JOneShot(\pi_1) = \JOneShot(\pi_2),
\qquad
\JSR(\pi_2) - \JSR(\pi_1) \geq \Delta(K),
\]
where $\Delta(K) > 0$.
\end{theorem}

\begin{proof}
\textbf{Construction.} Take $\Xcal$ supported on a single problem $x$ with $N \ge 2$ candidate answers $\{y^{(1)}, \ldots, y^{(N)}\}$ and ground-truth $y^\star$ drawn uniformly from this set; feedback is binary correctness. Define two history-independent policies, each sampling from a fixed distribution at every step:
\[
\pi_1(\cdot \mid \cdot) \;=\; \delta_{y^{(1)}}, \qquad
\pi_2(\cdot \mid \cdot) \;=\; \mathrm{Uniform}\{y^{(1)}, \ldots, y^{(N)}\},
\]
with $\pi_2$'s draws independent across steps.

\textbf{One-shot.} Since $y^\star$ is uniform,
\[
\JOneShot(\pi_1) \;=\; \Prob(y^\star = y^{(1)}) \;=\; \tfrac{1}{N}, \qquad
\JOneShot(\pi_2) \;=\; \sum_{i=1}^{N} \tfrac{1}{N} \cdot \tfrac{1}{N} \;=\; \tfrac{1}{N},
\]
so $\JOneShot(\pi_1) = \JOneShot(\pi_2)$.

\textbf{Sequential revision (budget $K$).} $\pi_1$ outputs $y^{(1)}$ at every step and therefore succeeds iff $y^\star = y^{(1)}$, giving $\JSR(\pi_1) = 1/N$. For $\pi_2$, the $K$ i.i.d.\ uniform draws are independent of $y^\star$, so the probability that none equals $y^\star$ is $(1 - 1/N)^K$, giving $\JSR(\pi_2) = 1 - (1 - 1/N)^K$.

\textbf{Gap.} For $K \ge 2$, monotonicity of $K \mapsto (1 - 1/N)^K$ on $K \ge 1$ gives $(1 - 1/N)^K \le (1 - 1/N)^2 = 1 - 2/N + 1/N^2$, hence
\[
\Delta(K) \;:=\; \JSR(\pi_2) - \JSR(\pi_1) \;=\; 1 - \bigl(1 - \tfrac{1}{N}\bigr)^{K} - \tfrac{1}{N} \;\ge\; \tfrac{N-1}{N^2} \;>\; 0. \qedhere
\]
\end{proof}

\subsection{Proof of Theorem~\ref{thm:whysrinformal}}
\label{appendix:proof-transfer}

The argument uses the recovery decomposition $J_\phi(\pi) = \sum_z \rho^\phi_\pi(z)\,V_\pi(z)$, which holds for any first-success transition strategy $\phi$ paired with any policy $\pi$. The proof of Lemma~\ref{lem:hazard} (Appendix~\ref{appendix:proof-hazard}) establishes this for $\phi = \SR$; the same three steps -- (i) writing the success event at step $t$ as $\{\tau \ge t,\, y_t \in G_x\}$, (ii) Fubini against $\sum_{t=1}^{K} \mathbf{1}\{\tau \ge t\} = \tau$, (iii) commuting the visit-count sum with the outer expectation -- depend only on the first-success stopping rule and the Markov property of $\pi$, so they generalize verbatim from $\rho^\SR_\pi$ to $\rho^\phi_\pi$ for any $\phi \in \PhiR$ that early-stops.

\begin{proof}[Proof of Theorem~\ref{thm:whysrinformal}]
By the recovery decomposition applied to $\pi_1$ and $\pi_0$,
\[
J_\phi(\pi_1) - J_\phi(\pi_0) = \sum_z \rho^\phi_{\pi_1}(z)\, V_{\pi_1}(z) - \sum_z \rho^\phi_{\pi_0}(z)\, V_{\pi_0}(z).
\]
Add and subtract $\sum_z \rho^\phi_{\pi_0}(z)\, V_{\pi_1}(z)$ and group:
\begin{equation}
\label{eq:transfer-split}
J_\phi(\pi_1) - J_\phi(\pi_0) = \underbrace{\sum_z \rho^\phi_{\pi_0}(z)\bigl(V_{\pi_1}(z) - V_{\pi_0}(z)\bigr)}_{(\star)} + \underbrace{\sum_z \bigl(\rho^\phi_{\pi_1}(z) - \rho^\phi_{\pi_0}(z)\bigr)\, V_{\pi_1}(z)}_{(\dagger)}.
\end{equation}

\textbf{Bounding $(\dagger)$.} Since $V_{\pi_1}(z) \in [0, 1]$,
\begin{equation}
\label{eq:transfer-dagger}
(\dagger) \;\ge\; -\sum_z \bigl|\rho^\phi_{\pi_1}(z) - \rho^\phi_{\pi_0}(z)\bigr| \;=\; -\bigl\|\rho^\phi_{\pi_1} - \rho^\phi_{\pi_0}\bigr\|_1.
\end{equation}

\textbf{Bounding $(\star)$.} We claim the pointwise inequality
\begin{equation}
\label{eq:transfer-pointwise}
\rho^\phi_{\pi_0}(z)\bigl(V_{\pi_1}(z) - V_{\pi_0}(z)\bigr) \;\ge\; \frac{1}{C_\phi}\,\rho^\SR_{\pi_0}(z)\bigl(V_{\pi_1}(z) - V_{\pi_0}(z)\bigr) \quad \text{for every } z.
\end{equation}

\emph{Case 1: $z \in \mathrm{supp}(\rho^\phi_{\pi_0})$.} By (C2), $V_{\pi_1}(z) - V_{\pi_0}(z) \ge 0$. By (C1), $\rho^\phi_{\pi_0}(z) \ge C_\phi^{-1}\,\rho^\SR_{\pi_0}(z) \ge 0$. Multiplying both sides of (C1) by the nonnegative quantity $V_{\pi_1}(z) - V_{\pi_0}(z)$ preserves the inequality and gives~\eqref{eq:transfer-pointwise}.

\emph{Case 2: $z \notin \mathrm{supp}(\rho^\phi_{\pi_0})$.} Then $\rho^\phi_{\pi_0}(z) = 0$, and (C1) forces $0 \ge C_\phi^{-1}\,\rho^\SR_{\pi_0}(z)$. Combined with $\rho^\SR_{\pi_0}(z) \ge 0$, this gives $\rho^\SR_{\pi_0}(z) = 0$, so both sides of~\eqref{eq:transfer-pointwise} vanish.

Summing~\eqref{eq:transfer-pointwise} over $z$,
\begin{equation}
\label{eq:transfer-star}
(\star) \;\ge\; \frac{1}{C_\phi}\sum_z \rho^\SR_{\pi_0}(z)\bigl(V_{\pi_1}(z) - V_{\pi_0}(z)\bigr).
\end{equation}

\textbf{Converting~\eqref{eq:transfer-star} to $\JSR(\pi_1) - \JSR(\pi_0)$.} Apply the recovery decomposition to $\phi = \SR$ at $\pi_1$ and $\pi_0$, then add and subtract $\sum_z \rho^\SR_{\pi_0}(z)\, V_{\pi_1}(z)$:
\[
\JSR(\pi_1) - \JSR(\pi_0) = \sum_z \rho^\SR_{\pi_0}(z)\bigl(V_{\pi_1}(z) - V_{\pi_0}(z)\bigr) + \sum_z \bigl(\rho^\SR_{\pi_1}(z) - \rho^\SR_{\pi_0}(z)\bigr)\, V_{\pi_1}(z).
\]
Rearranging and using $V_{\pi_1}(z) \in [0, 1]$,
\begin{equation}
\label{eq:transfer-srshift}
\sum_z \rho^\SR_{\pi_0}(z)\bigl(V_{\pi_1}(z) - V_{\pi_0}(z)\bigr) \;\ge\; \bigl[\JSR(\pi_1) - \JSR(\pi_0)\bigr] - \bigl\|\rho^\SR_{\pi_1} - \rho^\SR_{\pi_0}\bigr\|_1.
\end{equation}

\textbf{Combine.} Chaining~\eqref{eq:transfer-srshift} into~\eqref{eq:transfer-star} and adding~\eqref{eq:transfer-dagger},
\[
J_\phi(\pi_1) - J_\phi(\pi_0) \;\ge\; \frac{1}{C_\phi}\bigl[\JSR(\pi_1) - \JSR(\pi_0)\bigr] \;-\; \bigl\|\rho^\phi_{\pi_1} - \rho^\phi_{\pi_0}\bigr\|_1 \;-\; \frac{1}{C_\phi}\bigl\|\rho^\SR_{\pi_1} - \rho^\SR_{\pi_0}\bigr\|_1. \qedhere
\]
\end{proof}

\paragraph{Connection to our method.}
The third term $C_\phi^{-1}\,\|\rho^\SR_{\pi_1} - \rho^\SR_{\pi_0}\|_1$ in the bound is the price for placing the on-policy SR test-time objective $\JSR(\pi) = \sum_z \rho^\SR_\pi(z)\,V_\pi(z)$ on the right-hand side: as the policy improves, SR's own visit measure $\rho^\SR_\pi$ also shifts, and that shift contaminates $\JSR(\pi_1) - \JSR(\pi_0)$. The off-policy fixed-reference quantity
\[
J^\Rev_\SR(\pi) \;:=\; \sum_z \rho^\SR_{\pi_0}(z)\, V_\pi(z)
\]
absorbs this shift by holding the visit measure at the baseline; $J^\Rev_\SR(\pi_1) - J^\Rev_\SR(\pi_0) = \sum_z \rho^\SR_{\pi_0}(\cdot)(V_{\pi_1} - V_{\pi_0})$ exactly, and~\eqref{eq:transfer-star} terminates the proof in one step:
\[
J_\phi(\pi_1) - J_\phi(\pi_0) \;\ge\; \frac{1}{C_\phi}\bigl[J^\Rev_\SR(\pi_1) - J^\Rev_\SR(\pi_0)\bigr] - \bigl\|\rho^\phi_{\pi_1} - \rho^\phi_{\pi_0}\bigr\|_1.
\]
This off-policy form is exactly the training signal REVES optimizes (Section~\ref{sec:method}): Stage~I freezes the rollout distribution at the current $\pi_0$, generating SR trajectories whose visited states constitute samples from $\rho^\SR_{\pi_0}$; Stage~II then trains $V_\pi$ on those fixed states via single-turn RL. Each epoch's rollout refresh resets the baseline so the $\ell_1$ shift on $\rho^\SR$ stays small. The two forms are mathematically equivalent; the off-policy form makes the algorithmic interpretation transparent.

\paragraph{Interpretation.}
The bound is conditional, not unconditional. It bites when (i) $\phi$ covers SR's revision inputs (small $C_\phi$ in (C1)), (ii) the policy update raises one-step recovery on those states ((C2)), and (iii) the visit measures $\rho^\phi$ and $\rho^\SR$ are stable under the update (small $\ell_1$ shifts). The empirical confirmation that $\JSR$-trained policies improve other revision-using TTS algorithms (tree search, evolutionary refinement) is in Table~\ref{tab:aux_tts}.

\subsection{Proof of Lemma~\ref{lem:hazard}}
\label{appendix:proof-hazard}

\begin{proof}
We prove the three identities in~\eqref{eq:hazard-decomp} in order.

\emph{First identity.} SR stops at the first success, so $\rstar(x, y_\tau) = 1 \iff \exists\, t \le K : \rstar(x, y_t) = 1$, and the success event at step $t$ equals $\{\tau \ge t,\, y_t \in G_x\}$, where $G_x := \{y : \rstar(x, y) = 1\}$. By the tower property and the fact that $\pi_\theta$ is Markov in $z_t$,
\[
\Prob(\tau \ge t,\, y_t \in G_x) = \E\!\big[\mathbf{1}\{\tau \ge t\}\, \Prob(y_t \in G_x \mid z_t)\big] = \E\!\big[\mathbf{1}\{\tau \ge t\}\, V_\pi(z_t)\big].
\]
Summing over $t \le K$,
\[
\JSR(\theta) = \Prob\big(\exists\, t \le K : \rstar(x, y_t) = 1\big) = \sum_{t=1}^{K} \E\!\big[\mathbf{1}\{\tau \ge t\}\, V_\pi(z_t)\big].
\]

\emph{Second identity.} Using $\sum_{t=1}^{K} \mathbf{1}\{\tau \ge t\} = \tau$ and Fubini,
\[
\sum_{t=1}^{K} \E\!\big[\mathbf{1}\{\tau \ge t\}\, V_\pi(z_t)\big] = \E\!\Big[\sum_{t=1}^{\tau} V_\pi(z_t)\Big].
\]

\emph{Third identity.} Substituting the definition $V_\pi(z) = \E_{y' \sim \pi_\theta(\cdot \mid z)}[\rstar(x, y')]$ and exchanging the visit-count sum with the outer expectation gives
\[
\E\!\Big[\sum_{t=1}^{\tau} V_\pi(z_t)\Big] = \sum_z V_\pi(z)\, \E\!\Big[\sum_{t=1}^{\tau} \mathbf{1}\{z_t = z\}\Big] = \sum_z \rho_\theta(z)\, V_\pi(z) = \sum_z \rho_\theta(z)\, \E_{y' \sim \pi_\theta(\cdot \mid z)}\!\big[\rstar(x, y')\big]. \qedhere
\]
\end{proof}

% \newpage
\section{Training Prompt Template}
\label{appendix: training_prompt}

In this section, we present the training prompt templates used for data augmentation. An important distinction from the test-time setting is that we do not append execution error information to the context during training. Instead, as illustrated in Figure~\ref{fig:prompt_code_revision_binary} and Figure~\ref{fig:prompt_code_verification}, the training prompts rely on structured supervision rather than explicit error feedback.

% \vspace{-10em}

\begin{figure}[ht]
\centering
\begin{promptbox}{Prompt Template for Revision in Math Tasks}

\textbf{Problem} \\[0.3em]
\textcolor{brown}{
A frequency distribution of the scores for Mr.\ Sampson's algebra class is shown.
What percent of the class received a score in the $60\%-69\%$ range?
\begin{center}
\begin{tabular}{|c|c|}
\hline
Test Scores & Frequencies \\
\hline
$90\%-100\%$ & IIII \\
$80\%-89\%$  & IIII IIII \\
$70\%-79\%$  & IIII II \\
$60\%-69\%$  & IIII I \\
Below $60\%$ & II \\
\hline
\end{tabular}
\end{center}
}

\vspace{0.8em}
\textbf{Previous Answer} \\[0.3em]
\textcolor{brown}{
To solve this problem, we first count the number of students in each score range.
The total number of students is computed as $4 + 8 + 6 + 5 + 2 = 25$.
The number of students scoring in the $60\%-69\%$ range is $5$.
\dots
}

\vspace{0.8em}
\textbf{Verifiable Feedback:} \\
The previous answer is incorrect.

\vspace{0.8em}
\textbf{Instruction:} \\
Now produce a corrected solution. Identify and fix potential mistakes in the previous attempt.
Keep the reasoning concise and provide a corrected final answer.
Do not assume access to the correct answer.

\end{promptbox}
\caption{Prompt template for the revision stage with binary correctness feedback. Concrete problem content and partial incorrect answers are shown for illustration, with remaining content omitted for brevity.}
\label{fig:prompt_revision_binary}
\end{figure}

\begin{figure}[!ht]
\centering
\begin{promptbox}[unbreakable]{Prompt Template for Verification in Math Tasks}

\textbf{Problem} \\[0.3em]
\textcolor{brown}{
A frequency distribution of the scores for Mr.\ Sampson's algebra class is shown.
What percent of the class received a score in the $60\%-69\%$ range?
\begin{center}
\begin{tabular}{|c|c|}
\hline
Test Scores & Frequencies \\
\hline
$90\%-100\%$ & IIII \\
$80\%-89\%$  & IIII IIII \\
$70\%-79\%$  & IIII II \\
$60\%-69\%$  & IIII I \\
Below $60\%$ & II \\
\hline
\end{tabular}
\end{center}
}

\vspace{0.8em}
\textbf{Proposed Answer} \\[0.3em]
\textcolor{brown}{
To solve this problem, we first identify the number of students in each score range
from the frequency distribution.
The total number of students is computed as $4 + 8 + 6 + 5 + 2 = 25$.
Among them, $5$ students fall into 
\dots
}

\vspace{0.8em}
\textbf{Instruction:} \\
Judge whether the proposed answer is correct.
Respond concisely and end the final line with
\texttt{VERDICT: $\boxed{\text{True}}$} or
\texttt{VERDICT: $\boxed{\text{False}}$}.

\end{promptbox}
\caption{Prompt template used for the verification stage. The verifier observes the problem and a proposed solution and outputs a binary correctness judgment.}
\vspace{-0.5em}
\label{fig:prompt_verification}
\end{figure}

% \FloatBarrier
% \newpage

\section{Implementation Details}
\label{appendix:implementation}

In this section, we provide details of the training and evaluation implementations used in our experiments. All experiments are conducted on machines equipped with 8 $\times$ NVIDIA H200 GPUs.

% \subsection{Training Implementation}

\paragraph{RL Baseline.}
Our RL baseline is implemented using the official \texttt{Skywork-OR1} codebase.
We directly adopt their provided training scripts and configurations, which optimize a single-shot response using reinforcement learning.
No revision or multi-turn generation is involved during training for this baseline.

\paragraph{Multi-Turn and PAG Baselines.}
The Multi-Turn and PAG baselines are implemented based on the
\texttt{Policy-As-GenVerifier} codebase.
To ensure comparability with our setting, we make two modifications to the original implementation.
First, we increase the maximum response length to 8092 tokens to match our training setup.
Second, we adapt the reward computation function so that it correctly evaluates responses under our datasets, while keeping the reward definition itself identical to that used in our methods.
Apart from these changes, the original training pipelines and optimization settings are preserved.

\paragraph{REVES (Ours).}
Our method builds on top of the same training framework as the baselines.
The key difference lies in the data augmentation procedure.
Before training, we first perform an offline augmentation pass to initialize the augmented dataset.
During training, at the beginning of each epoch, we further augment the dataset by generating new revision trajectories using the current policy, and the newly generated samples are added to the training data.
When generating augmented data, we allow a maximum revision budget of $K=8$.
The training prompts used for augmentation follow the templates described in Appendix~\ref{appendix: training_prompt}.
The model is then trained on the mixture of original and augmented data using the same RL objective as the baseline.

% \subsection{Evaluation Setting}
% \label{appendix:evaluation}

\paragraph{Evaluation Protocol.}
Unless otherwise specified, all experiments are evaluated using our unified evaluation framework, which will be publicly released upon paper acceptance.
Given an input problem, the model follows the corresponding test-time protocol (e.g., sequential revision when applicable) to generate a final response.
The generated output is serialized into a JSON file and evaluated by an external, task-specific verifier.
This pipeline is applied uniformly across all tasks and methods to ensure fair comparison: generation and evaluation are fully decoupled, and correctness is determined solely by the external verifier rather than any internal model signals.
During evaluation, the maximum response length is set to 16{,}384 tokens for coding tasks and AIME24/25, and 8{,}096 tokens for MATH500.
The decoding temperature is fixed to 0.6 for coding tasks and 1.0 for all other tasks, and the confidence threshold in the SelfConf setting is set to $c=0.5$.

For math tasks, when evaluating the Multi-Turn and PAG baselines, we follow the original evaluation implementation provided in the \texttt{Policy-As-GenVerifier} repository to remain faithful to their protocol.
Under this evaluation pipeline, we use the same maximum response length settings for 7B models as described above.
For 3B models with revision budget $K=4$, where the effective sequence length exceeds the model context limit, we reduce the maximum response length to 4{,}096 tokens.
All math evaluations are conducted using \texttt{math\_eval}, ensuring that results are directly comparable across evaluation settings.

\section{Method Details: Multi-Turn Comparison, Efficiency, and Test-Time Sensitivity}
\label{appendix:method-details}

\subsection{Comparison with Naive Multi-Turn RL}
\label{appendix:multiturn-comparison}

A natural alternative to our offline scheme is online sequential revision
with standard multi-turn credit assignment. Two such schemes appear in the
recent literature. The first, \emph{trajectory-level broadcast}, samples $G$
trajectories of $M$ turns per prompt, scores each trajectory by its final
reward, computes a trajectory-level advantage, and broadcasts that single
advantage uniformly to every turn in the trajectory. The second, \emph{pooled
per-response}, pools all $G \times M$ responses across the $G$ trajectories,
treats each response as an independent sample, and computes advantages within
the pool.

Neither scheme suits our setting. Sequential revision is structurally
different from a vanilla multi-turn task: each turn is a fresh attempt at the
\emph{same} problem $x$, not a step in a longer-horizon decision process with
intermediate sub-rewards. Trajectory-level broadcast assigns identical credit
to every turn and therefore ignores that some intermediate responses are much
closer to correctness than others. Pooled per-response also fails to capture
this structure: the wrong intermediate responses are treated identically,
even though the near-miss failures are exactly the ones we most want the
policy to learn from. Our offline selection-then-replay scheme bypasses both
pathologies. We identify near-miss failures explicitly through the
success-trajectory filter and replay them as single-turn revision prompts,
where standard single-turn credit assignment applies cleanly.

\subsection{Wall-Clock Efficiency}
\label{appendix:efficiency}

Beyond avoiding credit-assignment pitfalls, our offline scheme is
substantially faster than online multi-turn training. Multi-turn RL pays a
length-$M$ serial sequential-revision cost per prompt at every gradient step,
multiplied by the $G$ rollouts needed to estimate the advantage. Our method
pays the length-$K$ serial cost \emph{only once} per prompt at the offline
stage, and during the online RL stage performs only single-turn sampling on
the extracted prompts, which can be parallelized across $N$ Best-of-$N$
samples. Single-turn parallel sampling is much cheaper per gradient step than
$M$-turn serial generation across $G$ rollouts, since serial generation does
not benefit from the parallelism that modern inference engines provide for
single-turn sampling.

Empirically, on $8\times$ NVIDIA H200 GPUs, three full iterations of our
method completed in roughly 30 hours, while the equivalent online multi-turn
baseline required over 48 hours on identical hardware
(Table~\ref{tab:efficiency}). With a more carefully optimized
infrastructure, Stage~I data generation and Stage~II RL training could run
concurrently (Stage~II first trains on the original RL data while Stage~I
generates augmented data asynchronously), yielding further gains.

\begin{table}[h]
\centering
\caption{Per-iteration wall-clock breakdown of our method on $8\times$ NVIDIA H200 GPUs.}
\label{tab:efficiency}
\small
\begin{tabular}{lccc}
\toprule
Iteration & Data Generation (Stage~I) & RL Training (Stage~II) & Total \\
\midrule
Iter 1 & 3h 4m   & 7h 15m  & 10h 19m \\
Iter 2 & 2h 41m  & 7h 26m  & 10h 7m  \\
Iter 3 & 3h 13m  & 6h 23m  & 9h 36m  \\
\midrule
Total  & 8h 58m  & 21h 4m  & 30h 2m  \\
\bottomrule
\end{tabular}
\end{table}

\subsection{Sensitivity to the Confidence Threshold}
\label{appendix:threshold-sensitivity}

We study the sensitivity of our test-time stopping rule to the confidence
threshold $c$ used in the \texttt{TailConfidence} stopping criterion
(Section~\ref{sec:method}). Table~\ref{tab:threshold-sensitivity} reports
accuracy and the average number of revision rounds used at budget $K = 16$
across three math benchmarks. Higher thresholds enforce stricter stopping
(more rounds before accepting a response), and accuracy increases
monotonically with the threshold over the range we tested. The number of
rounds also grows with the threshold, reflecting the cost-quality trade-off.

\begin{table}[h]
\centering
\caption{Sensitivity of test-time accuracy to the \texttt{TailConfidence} threshold $c$ on Qwen-2.5-7B-REVES at budget $K = 16$. ``Rounds'' is the average number of revision rounds before stopping.}
\label{tab:threshold-sensitivity}
\small
\begin{tabular}{l cc cc cc}
\toprule
\multirow{2}{*}{\textbf{Threshold $c$}}
 & \multicolumn{2}{c}{\textbf{AIME25}}
 & \multicolumn{2}{c}{\textbf{AIME24}}
 & \multicolumn{2}{c}{\textbf{MATH500}} \\
 \cmidrule(lr){2-3} \cmidrule(lr){4-5} \cmidrule(lr){6-7}
 & Acc.\ (\%) & Rounds & Acc.\ (\%) & Rounds & Acc.\ (\%) & Rounds \\
\midrule
$0.4$ & $16.56$ & $10.22$ & $21.15$ & $9.04$  & $77.80$ & $3.99$ \\
$0.5$ & $16.88$ & $11.91$ & $21.56$ & $10.06$ & $79.17$ & $4.46$ \\
$0.6$ & $17.50$ & $13.31$ & $22.08$ & $12.19$ & $79.10$ & $5.64$ \\
$0.7$ & $18.23$ & $15.45$ & $23.65$ & $14.97$ & $81.58$ & $8.76$ \\
\bottomrule
\end{tabular}
\end{table}

\section{Additional Experiments and Ablations}
\label{appendix:strong}
We additionally evaluate our method using a strong reasoning-oriented base model,
DeepSeek-R1-Distill-7B \citep{deepseekai2025deepseekr1incentivizingreasoningcapability}, and compare against both the original base model and an RL baseline.
Due to the substantially longer responses produced by this model
(approximately $9\text{k}$ tokens on average, compared to $\sim 2\text{k}$ tokens for Qwen-7B),
multi-turn training becomes significantly more expensive, which limits the scale of these experiments.

As shown in Table~\ref{tab:lcb_two_periods}, our REVES model consistently improves test-case-based performance over both the base model and the RL baseline. We expect larger gains with increased data scale and longer training.

\begin{table}
\centering
\caption{Performance on LiveCodeBench under different evaluation periods.}
\label{tab:lcb_two_periods}

\setlength{\tabcolsep}{6pt}
\renewcommand{\arraystretch}{1.1}

\resizebox{0.9\textwidth}{!}{
\begin{tabular}{l cc cc}
\toprule
\multirow{2}{*}{\textbf{Method}}
& \multicolumn{2}{c}{\textbf{LiveCodeBench (01/01/25--05/01/25)}}
& \multicolumn{2}{c}{\textbf{LiveCodeBench (08/01/24--01/01/25)}} \\
\cmidrule(lr){2-3} \cmidrule(lr){4-5}

& \textbf{OneShot} & \textbf{TestCases-32}
& \textbf{OneShot} & \textbf{TestCases-32} \\
\midrule

Base (Distill-7B)
& 34.25 & 44.51
& 38.24 & 52.21 \\

Base + RL
& \textbf{36.63} & 43.96
& 38.79 & 53.19 \\

Base + REVES
& 35.53 & \textbf{46.33}
& \textbf{40.63} & \textbf{54.04} \\

\bottomrule
\end{tabular}
}
\end{table}

% MATH500 ablation results were merged into the main-text 4-panel figure
% (Figure~\ref{fig:ablation-4panel}, panels (c) and (d)).

\section{Circle Packing Implementation Details}
\label{appendix:circle_packing}

This section gives the implementation details behind the circle packing results in Table~\ref{tab:circle-packing}. Circle packing maximizes the sum of radii of $n$ non-overlapping circles inscribed in the unit square, with verifier feedback given by a non-overlap and inscription check together with the realized objective value. We instantiate REVES on this benchmark by directly optimizing $\JSR$ in the recovery form of Lemma~\ref{lem:hazard}: each visited state $z_t$ corresponds to a partially-improved candidate configuration, the policy proposes a refined configuration $y_t \sim \pi_\theta(\cdot \mid z_t)$, and the verifier returns $\rstar(x, y_t) = 1$ iff $y_t$ improves on the running best while remaining feasible.

\paragraph{Reward shaping.}
The raw objective of a candidate $y$ is its sum of radii $g(y) := \sum_{i=1}^{n} r_i(y)$ (a continuous quantity), which clusters in a narrow band near the optimum ($g \approx 2.63$ for the $n=26$ instance we evaluate) and differs across candidates only at the third or fourth decimal. Used directly, this signal is too compressed to drive learning. We therefore shape it in two steps. Let $v(y) \ge 0$ denote the total constraint violation (circle overlap plus boundary excess), so that $y$ is feasible iff $v(y) = 0$. First, define the raw score $s(y)$ as
\[
s(y) =
\begin{cases}
g(y), & v(y) = 0, \\
g(y) - v(y), & v(y) > 0 \ \text{and}\ g(y) \le 2.64, \\
-1, & v(y) > 0 \ \text{and}\ g(y) > 2.64,
\end{cases}
\]
i.e., minor violations are absorbed as a soft penalty, while implausibly high sums coupled with infeasibility are flagged as cheating and assigned $s = -1$. Second, we amplify by squaring the offset from $2.5$: $R(s) := (s - 2.5)^2$. Since feasible raw scores live above $2.5$, this magnifies small absolute differences, e.g., a $\Delta s = 10^{-3}$ gain near $s = 2.6$ produces a $\Delta R \approx 2 \cdot 10^{-4}$, which is no longer dwarfed by gradient noise. The reward used for a revision $y$ from parent state $x$ is the resulting improvement,
\[
U(y \mid x) := R(s(y)) - R(s(x)),
\]
so the policy is rewarded for any move that improves on the parent and penalized for any move that worsens it.

\paragraph{Stage~I: search rollouts and prompt-set construction.}
We run $4$ sequential-revision rollouts under the current $\pi_\theta$ with budget $K = 100$, where each step proposes a refinement of the previous candidate configuration and the verifier returns the shaped reward $U(y \mid x)$ defined above. From the visited intermediate states across these rollouts, we filter to retain those whose subsequent revisions produce verified improvements, and assemble the resulting revision prompts into the augmented prompt set.

\paragraph{Stage~II: RL training.}
We train Qwen3-4B with GRPO~\citep{deepseekai2025deepseekr1incentivizingreasoningcapability} on the prompt set from Stage~I. Each gradient step samples $8$ prompts; for each prompt we draw $64$ rollouts to form the group used for the GRPO advantage estimate.

\paragraph{A representative solution.}
Figure~\ref{fig:circle-packing-solution} visualizes one of the configurations REVES discovers, and Figure~\ref{lst:circle-packing-code} reproduces the Python program emitted by the policy that generates it: an initial grid placement followed by an SLSQP refinement against the non-overlap and inscription constraints.

\begin{figure}[h]
\centering
\includegraphics[width=0.7\linewidth]{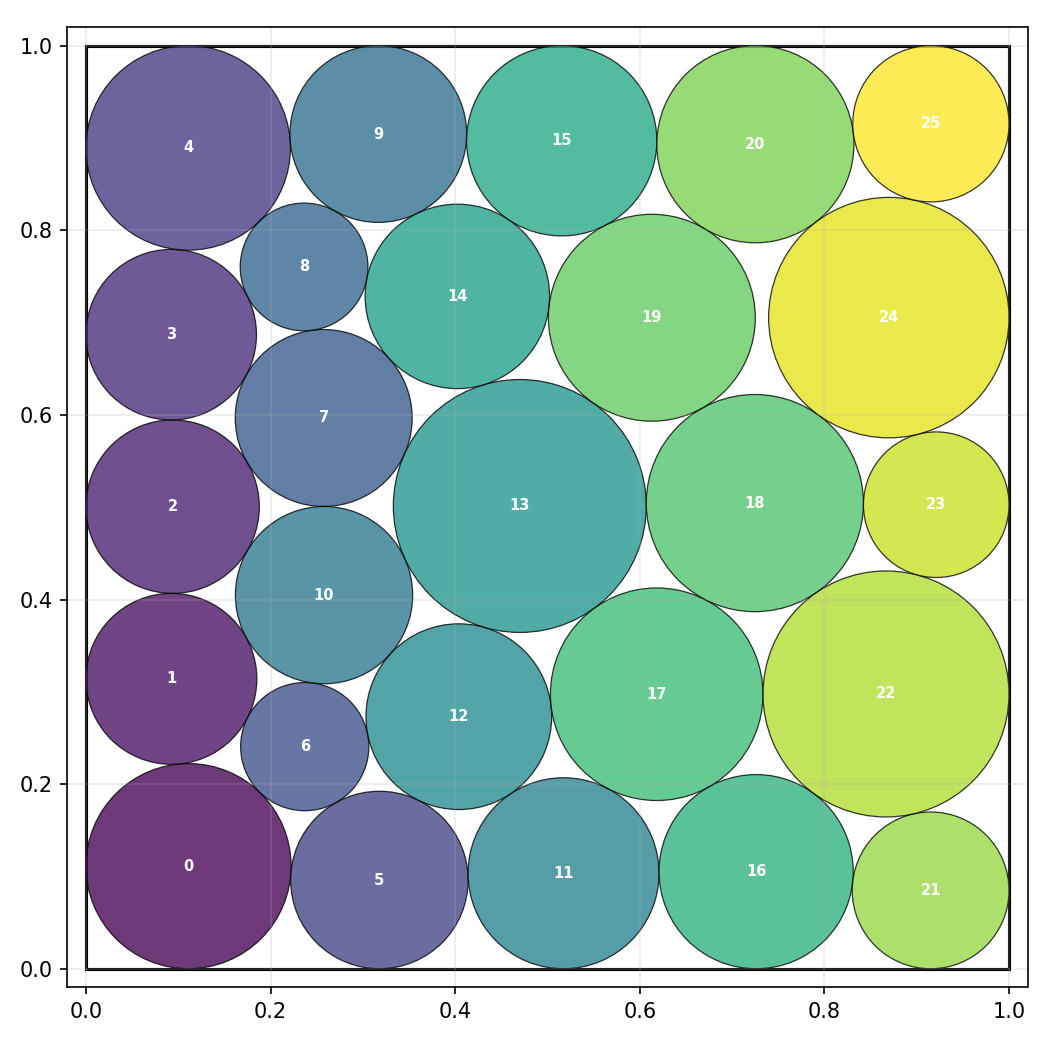}
\caption{A circle packing configuration found by REVES on the $n=26$ instance, achieving sum-of-radii $=2.635983$.}
\label{fig:circle-packing-solution}
\end{figure}

\begin{codebox}{Circle-Packing Program}
import numpy as np
from scipy.optimize import minimize


def run_packing():
    num_circles = 26
    r_initial = 0.08
    x_positions = np.linspace(r_initial, 1 - r_initial, 6)
    y_positions = np.linspace(r_initial, 1 - r_initial, 5)
    positions = [(x, y) for x in x_positions for y in y_positions][:25]
    positions.append((0.3, 0.3))
    positions = np.unique(positions, axis=0)
    if len(positions) < 26:
        positions = np.append(positions, [(0.5, 0.5)] * (26 - len(positions)), axis=0)

    radii = np.full(num_circles, r_initial)

    variables = np.zeros(num_circles * 3)
    for i in range(num_circles):
        variables[2 * i] = positions[i][0]
        variables[2 * i + 1] = positions[i][1]
    for i in range(num_circles):
        variables[2 * num_circles + i] = radii[i]

    def objective(vars):
        return -np.sum(vars[num_circles * 2:])

    constraints = []
    for i in range(num_circles):
        def con1(vars, i=i):
            centers = vars[:num_circles * 2].reshape(num_circles, 2)
            r = vars[2 * num_circles + i]
            return centers[i, 0] - r
        constraints.append({'type': 'ineq', 'fun': con1})

        def con2(vars, i=i):
            centers = vars[:num_circles * 2].reshape(num_circles, 2)
            r = vars[2 * num_circles + i]
            return 1 - centers[i, 0] - r
        constraints.append({'type': 'ineq', 'fun': con2})

        def con3(vars, i=i):
            centers = vars[:num_circles * 2].reshape(num_circles, 2)
            r = vars[2 * num_circles + i]
            return centers[i, 1] - r
        constraints.append({'type': 'ineq', 'fun': con3})

        def con4(vars, i=i):
            centers = vars[:num_circles * 2].reshape(num_circles, 2)
            r = vars[2 * num_circles + i]
            return 1 - centers[i, 1] - r
        constraints.append({'type': 'ineq', 'fun': con4})

    for i in range(num_circles):
        def con5(vars, i=i):
            return vars[2 * num_circles + i]
        constraints.append({'type': 'ineq', 'fun': con5})

    for i in range(num_circles):
        for j in range(i + 1, num_circles):
            def con6(vars, i=i, j=j):
                centers = vars[:num_circles * 2].reshape(num_circles, 2)
                r_i = vars[2 * num_circles + i]
                r_j = vars[2 * num_circles + j]
                dx = centers[i, 0] - centers[j, 0]
                dy = centers[i, 1] - centers[j, 1]
                return dx**2 + dy**2 - (r_i + r_j)**2
            constraints.append({'type': 'ineq', 'fun': con6})

    result = minimize(
        objective,
        variables,
        constraints=constraints,
        method='SLSQP',
        options={'maxiter': 2000000, 'ftol': 1e-10, 'disp': False}
    )

    centers = result.x[:num_circles * 2].reshape(num_circles, 2)
    radii = result.x[num_circles * 2:]
    return centers, radii, float(np.sum(radii))
\end{codebox}
\captionof{figure}{Python program output by the REVES-trained Qwen3-4B policy that produces the configuration in Figure~\ref{fig:circle-packing-solution}.}
\label{lst:circle-packing-code}

\begin{promptbox}{Prompt Template for Code Revision}

\textbf{Problem} \\[0.3em]
\textcolor{brown}{
You will be given a problem specification and are asked to generate a correct Python
program that satisfies the requirements and passes all tests.
The task involves extracting two integers and an operator (``gains'' or ``loses'')
from a natural language string and returning the result of the corresponding calculation.
\dots
}

\vspace{0.8em}
\textbf{Previous Answer} \\[0.3em]
\textcolor{brown}{To address the issue in the provided code, there are a few key points ...}
\begin{tcolorbox}[
  colback=brown!8!white,
  colframe=brown!60!black,
  boxrule=0.2mm,
  arc=2mm,
  left=1mm,right=1mm,top=1mm,bottom=1mm
]
\begin{verbatim}
class Solution:
    def calculate(self, string: str) -> int:
        words = string.split()
        ...
        joined_string = "".join(words)
        result = eval(joined_string)
        return result
\end{verbatim}
\end{tcolorbox}
\textcolor{brown}{\dots}

\vspace{0.8em}
\textbf{Verifiable Feedback:} \\
The previous answer is incorrect.

\vspace{0.8em}
\textbf{Instruction:} \\
Please repair the code accordingly.
Enclose your fixed implementation within the specified Python code delimiters.
Do not assume access to the correct implementation.

\end{promptbox}
\captionof{figure}{Prompt template for the code revision stage with binary correctness feedback. Partial incorrect code is shown for illustration, with remaining content omitted for brevity.}
\label{fig:prompt_code_revision_binary}

\bigskip

\begin{promptbox}{Prompt Template for Code Verification}

\vspace{0.8em}
\textbf{Problem} \\[0.3em]
\textcolor{brown}{
You will be given a problem specification and are asked to generate a correct Python
program that satisfies the requirements and passes all tests.
The task involves extracting two integers and an operator (``gains'' or ``loses'')
from a natural language string and returning the result of the corresponding calculation.
\dots
}

\vspace{0.8em}
\textbf{Candidate Answer:} \\[0.3em]
\textcolor{brown}{To address the issue in the provided code, there are a few key points ...}
\begin{tcolorbox}[
  colback=brown!8!white,
  colframe=brown!60!black,
  boxrule=0.2mm,
  arc=2mm,
  left=1mm,right=1mm,top=1mm,bottom=1mm
]
\begin{verbatim}
class Solution:
    def calculate(self, string: str) -> int:
        words = string.split()
        ...
        joined_string = "".join(words)
        result = eval(joined_string)
        return result
\end{verbatim}
\end{tcolorbox}
\textcolor{brown}{\dots}

\vspace{0.8em}
\textbf{Instruction:} \\[0.3em]
Please verify whether the candidate answer is correct.  
At the end, output the final line exactly as:
\texttt{VERDICT: \textbackslash boxed\{True/False\}}

\end{promptbox}
\captionof{figure}{Prompt template for the code verification stage. The verifier evaluates a candidate solution and outputs a binary correctness verdict.}
\label{fig:prompt_code_verification}

\end{document}